%% file: main.tex
\documentclass{article}

\PassOptionsToPackage{numbers, sort&compress}{natbib}
\usepackage[final, main, nonatbib]{neurips_2026}

\usepackage{microtype}
\usepackage{graphicx}
\usepackage{subcaption}
\usepackage{wrapfig}
\usepackage{booktabs}
\usepackage{multirow}
\usepackage[table]{xcolor}
\usepackage[utf8]{inputenc}
\usepackage[T1]{fontenc}
\usepackage[numbers]{natbib}
\usepackage{hyperref}
\usepackage{url}
\usepackage{amsfonts}
\usepackage{nicefrac}

\definecolor{mydarkblue}{rgb}{0,0.08,0.45}
\hypersetup{
  colorlinks=true,
  linkcolor=mydarkblue,
  citecolor=mydarkblue,
  urlcolor=mydarkblue,
  filecolor=mydarkblue,
}

\definecolor{cvprblue}{RGB}{31,119,180}
\hypersetup{
  linkcolor=cvprblue,
  citecolor=cvprblue,
  urlcolor=cvprblue,
  filecolor=cvprblue,
}
\definecolor{darkgreen}{RGB}{34,139,34}

\usepackage{amsmath}
\usepackage{amssymb}
\usepackage{mathtools}
\usepackage{amsthm}
\usepackage{etoolbox}

\AtBeginEnvironment{equation}{\small}
\AtBeginEnvironment{equation*}{\small}
\AtBeginEnvironment{align}{\small}
\AtBeginEnvironment{align*}{\small}
\usepackage{pifont}
\usepackage{bm}
\usepackage{cancel}
\usepackage[most]{tcolorbox}

\newcommand{\yrcite}[1]{\cite{#1}}

\usepackage[capitalize,noabbrev]{cleveref}

\theoremstyle{plain}

\newtheorem*{lemma*}{Lemma}
\newtheorem{corollary}{Corollary}
\newtheorem*{corollary*}{Corollary}
\theoremstyle{definition}
\newtheorem{definition}{Definition}

\newtheorem{property}{Property}
\newtheorem*{property*}{Property}

\newtheoremstyle{heuristicstyle}
  {0.5em}
  {0.5em}
  {}
  {}
  {\bfseries}
  {.}
  {0.5em}
  {\thmname{#1}\thmnumber{ #2}\thmnote{ (\textit{#3})}}
\theoremstyle{heuristicstyle}
\newtheorem{remark}{Remark}
\newtheorem*{remark*}{Remark}

\newtheorem*{observation*}{Observation}

\newtheorem*{insight*}{Insight}

\newcommand{\papertitle}{When Policy Entropy Constraint Fails: Preserving Diversity in Flow-based RLHF via Perceptual Entropy}

\title{\papertitle}

\author{
  \bfseries
  Xiaofeng Tan$^{1,2,*}$\quad
  Jun Liu$^{2}$\quad
  Bin-Bin Gao$^{2}$\quad
  Yuanting Fan$^{2}$\quad
  Xi Jiang$^{3}$\\
  \bfseries
  Chengjie Wang$^{2}$\quad
  Hongsong Wang$^{1,\dagger}$\quad
  Feng Zheng$^{3,\dagger}$ \\[0.4em]
  \normalfont
  $^{1}$Southeast University\quad
  $^{2}$Tencent YouTu Lab\quad
  $^{3}$Southern University of Science and Technology \\[0.2em]
  \normalfont \texttt{\{xiaofengtan, hongsongwang\}@seu.edu.cn, f.zheng@ieee.org}
}

\begin{document}

\maketitle
\begingroup
\renewcommand{\thefootnote}{\fnsymbol{footnote}}
\footnotetext[1]{Work done during an internship at Tencent YouTu Lab.\quad \textsuperscript{$\dagger$}Corresponding authors.}
\endgroup
\addtocontents{toc}{\protect\setcounter{tocdepth}{-1}}

\begin{abstract}
  \input{sections/0_abstract}

\end{abstract}

\input{sections/1_introduction}
\input{sections/2_preliminary}

\input{sections/3_motivation}

\input{sections/4_method}

\input{sections/5_experiments}

\input{sections/7_related_works}

\input{sections/6_conclusion}

\begin{ack}
This work was supported by the National Natural Science Foundation of China (62672109, 62302093, 52441503), the Jiangsu Province Natural Science Fund (BK20230833), and the Open Research Fund of the State Key Laboratory of Multimodal Artificial Intelligence Systems under Grant E5SP060116. We thank the Big Data Computing Center of Southeast University for providing the facility support on the numerical calculations.
\end{ack}

\bibliographystyle{plainnat}
\bibliography{main}

\onecolumn
\appendix

\renewcommand{\theequation}{A\arabic{equation}}
\renewcommand{\thefigure}{A\arabic{figure}}
\renewcommand{\thetable}{A\arabic{table}}
\renewcommand{\thesection}{\Alph{section}}

\renewcommand{\theHequation}{app.A\arabic{equation}}
\renewcommand{\theHfigure}{app.A\arabic{figure}}
\renewcommand{\theHtable}{app.A\arabic{table}}
\renewcommand{\theHsection}{app.\Alph{section}}

\setcounter{equation}{0}
\setcounter{figure}{0}
\setcounter{table}{0}

\hypersetup{
  linkcolor=cvprblue,
  citecolor=cvprblue,
  urlcolor=cvprblue,
  filecolor=cvprblue,
}
\addtocontents{toc}{\protect\setcounter{tocdepth}{2}}
\setcounter{tocdepth}{2}
\section*{Appendix}
This appendix presents additional experimental results, including extended visualizations and user studies in Sec.~\ref{app:visualization}, reward variance analysis in Sec.~\ref{app:reward_variance}, and rule-based reward results on SD3.5-M in Sec.~\ref{app:sd35_rule_based}. It further provides theoretical proofs and derivations in Sec.~\ref{app:theoretical}, followed by implementation details of covariance-aware entropy control in Sec.~\ref{app:cov-flow}, and broader impact in Sec.~\ref{app:broader_impact}.
\tableofcontents

\input{sections/8_appendix}

\end{document}

%% file: sections/0_abstract.tex
RLHF is widely used to align flow-matching text-to-image models with human preferences, but often leads to severe diversity collapse after fine-tuning. In RL, diversity is often assumed to correlate with policy entropy, motivating entropy regularization. However, we show this intuition breaks in flow models: policy entropy remains constant, even while perceptual diversity collapses. We explain this mismatch both theoretically and empirically: the constant entropy arises from the fixed, pre-defined noise schedule, while the diversity collapse is driven by the mode-seeking nature of policy gradients. As a result, policy entropy fails to prevent the model from converging to a narrow high-reward region in the perceptual space. To this end, we introduce perceptual entropy that captures diversity in a perceptual space and maintains the property of standard entropy. Building upon this insight, we propose two entropy-regularized strategies, \textit{Perceptual Entropy Constraint} (\textit{Perc-EC}) and \textit{Perceptual VAE Constraint} (\textit{Perc-VC}), to preserve perceptual diversity and improve generation quality. Experiments across two base models, neural and rule-based rewards, and two perceptual spaces demonstrate consistent gains in the quality-diversity trade-off. Our \href{https://xiaofeng-tan.github.io/projects/PEC/index.html}{project page} with code is publicly available.

%% file: sections/1_introduction.tex
\section{Introduction}
\label{sec:introduction}

Flow matching~\citep{lipman2022flow} has emerged as a powerful paradigm for text-to-image generation~\citep{wu2025qwenimagetechnicalreport}. Since human-centric criteria such as aesthetics~\citep{discus0434_aesthetic_2024}, semantic consistency~\citep{radford2021learning}, and safety~\citep{liu2024alignguard} are difficult to optimize with likelihood-based pre-training alone, recent work aligns pretrained generative models through Reinforcement Learning from Human Feedback (RLHF)~\citep{fan2023dpok,wallace2024diffusion,liu2025flow}.

RLHF addresses this by post-training generative models with feedback from human annotators or reward models. However, diversity can decrease during fine-tuning, which limits both \textit{\textbf{(1)}} RLHF effectiveness and \textit{\textbf{(2)}} real-world applicability. Specifically, \textit{\textbf{(1)}} RLHF requires diverse and informative samples for reward evaluation; reduced diversity narrows exploration and weakens feedback scope. \textit{\textbf{(2)}} Many applications require coverage across contexts \cite{albuquerque2025benchmarking, dombrowski2025image}. As shown in Fig.~\ref{fig:teaser}, Flow-GRPO produces visually similar samples for both portrait and living-room prompts, while an ideal model should preserve variations in identity, pose, background, layout, lighting, and furnishing style. This motivates us to seek RLHF frameworks that \textit{improve quality without sacrificing diversity}.

\begin{figure}[t]
    \centering
    \includegraphics[width=1.0\textwidth]{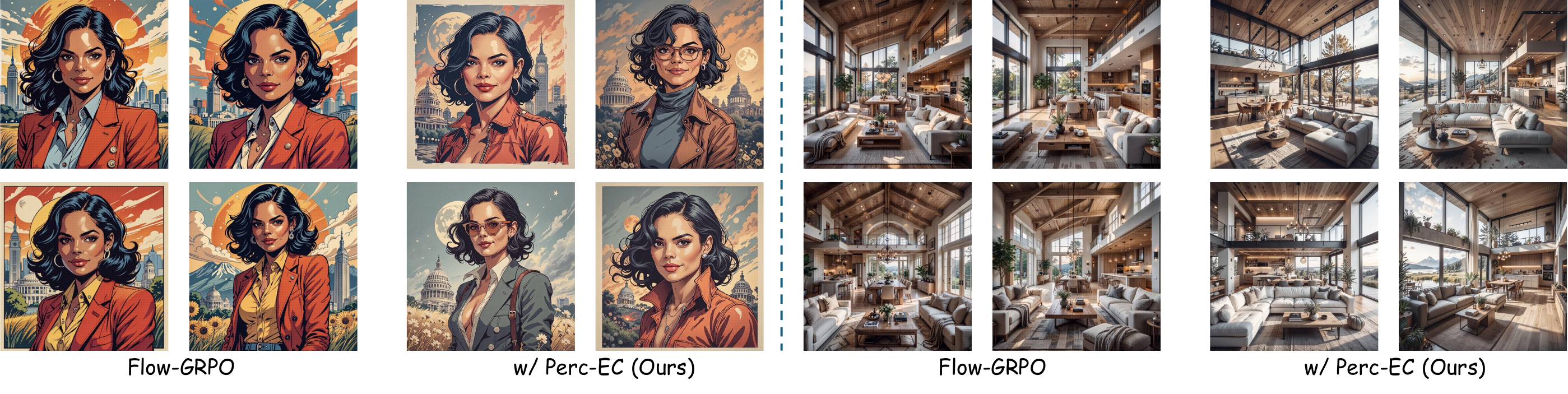}
    \vspace{-1.3em}
    \caption{\textbf{Diversity Collapse.} Flow-GRPO improves quality but sacrifices diversity.}
       \label{fig:teaser}
       \vspace{-1.0em}
\end{figure}

In RL research on LLMs, diversity is commonly characterized by policy entropy \citep{zhu2506surprising, jiang2025rethinking, shen2025exploring, cheng2025reasoning,liu2025prorl}.  Under this view, recent works \cite{yue2025does, cui2025entropy} aim to balance diversity and reward $\mathcal{R}$ using policy entropy $\exp(\mathcal{H})$, and reports an empirical relationship, $\mathcal{R} = -a \exp(\mathcal{H}) + b,$ with coefficients $a$ and $b$.  A natural solution to diversity collapse in flow-based RLHF is therefore policy entropy regularization.

However, revisiting this intuition under flow matching reveals a key inconsistency: \textit{diversity collapses while the policy entropy surrogate remains constant}, as shown in Fig.~\ref{fig:motivation}. Diffusion and flow-based models~\citep{DDPM_NIPS2020, song2021scorebasedsde} avoid directly modeling $p(\mathbf{x}_0)$ by decomposing the process into transitions $p(\mathbf{x}_{t-1} \mid \mathbf{x}_{t})$, whose entropy we term policy entropy. However, this quantity is fixed by the noise schedule and thus fails to reflect changes in diversity. This raises a fundamental question: \textit{what drives diversity collapse in RL for flow models despite constant policy entropy?}

\textbf{Contributions.} In this paper, we investigate the diversity collapse from the perspective of entropy mechanism in flow-based RLHF, and propose a novel entropy concept with two entropy-regularization strategies to improve quality without sacrificing diversity. Main contributions are highlighted below.

First, \textbf{we investigate the apparent paradox} between constant entropy and diversity collapse in flow-based RLHF. We study two questions: {\emph{\textbf{(1)} why does the policy entropy remain constant after RL fine-tuning?}} and \emph{\textbf{(2)} why can diversity still collapse under this constant policy entropy?}  {\textit{\textbf{(1)}}} For the first question, we show that the noised samples $\mathbf{x}_t$ are generated by injecting noise with a \emph{pre-defined and fixed intensity} into the model-predicted mean $\mu_\theta(\mathbf{x}_{t})$, resulting in a Gaussian distribution with fixed covariance. Hence, both likelihoods and policy entropy are determined by the fixed noise scheduler. {\textit{\textbf{(2)}}} For the second question, we find that this paradox stems from the mode-seeking nature of online RL, where the policy is optimized to concentrate probability distribution on high-reward region in the reward space. In this case, the multi-peak pretrained policy distribution may collapse into a single-peak perceptual distribution. This motivates an RLHF framework that encourages broad coverage of multiple high-reward regions in perceptual space, rather than convergence to a single peak, as shown in Fig.~\ref{fig:mode_mean_seeking}.

Second,  motivated by the observation that policy entropy in the generative space fails to capture diversity collapse, \textbf{we introduce a novel entropy concept}, {\textit{perceptual entropy}} $\mathcal{H}_{\text{perc}}$, which faithfully reflects this phenomenon while remaining consistent with prior entropy-based methods. Building on this concept, we further reveal an empirical entropy-based reward for flow models,
$\mathcal{R} = -a \exp(\mathcal{H}_{\text{perc}}) + b,$
which is analogous to relationships observed in LLMs~\citep{cui2025entropy}. Furthermore, building on the above analysis, \textbf{we propose two entropy-regularized strategies} to mitigate diversity collapse: \textit{Perceptual Entropy Constraint} (\textit{Perc-EC}), which incorporates perceptual entropy into the reward signal, and \textit{Perceptual VAE Constraint} (\textit{Perc-VC}), which aligns the perceptual distribution with the diversity-preserving VAE space to maintain diversity.

Finally, comprehensive experiments on FLUX.dev~\cite{flux2024} and SD3.5-M~\cite{sd3} using neural and rule-based rewards, alongside PickScore~\cite{kirstain2023pick}, DINO~\cite{DBLP:journals/corr/abs-2104-14294}, and CLIP~\cite{radford2021learning} perceptual encoders, demonstrate that perceptual entropy improves the quality-diversity trade-off.

%% file: sections/2_preliminary.tex
\section{Preliminary}
\label{sec:preliminary}

\newcommand{\bx}{\mathbf{x}}
\newcommand{\bz}{\mathbf{z}}

\textbf{RLHF Objective and Policy Gradient.} Given an external prompt condition $c\sim\mathcal{D}$ and a policy $p_{\theta}$ that generates $\bx\sim p_{\theta}(\cdot\mid c)$, RLHF aims to maximize the expected reward:
\begin{equation}\label{eq:rlhf}
    \max_{\theta}\; \mathcal{J}(\theta)
    := \mathbb{E}_{c\sim \mathcal{D},\, \bx\sim p_{\theta}(\cdot\mid c)}\bigl[\mathcal{R}(\bx, c)\bigr],
\end{equation}

Considering that the gradient of the objective is difficult to estimate directly, policy gradient methods such as REINFORCE~\citep{williams1992simple} decompose it into a sum of log-likelihood gradients:
\begin{equation}
\label{eq:pg}
    \nabla_{\theta} \mathcal{J}(\theta) = \mathbb{E}\Bigl[\sum_{t=1}^{T} \nabla_{\theta}\log p_{\theta}(\bx_{t-1}\mid \bx_t,c)\, \mathcal{A}(\bx,c)\Bigr],
\end{equation}
where $\mathcal{A}(\bx,c)$ denotes an advantage estimated by the final sample $\bx\equiv \bx_0$.

\textbf{Group-Relative Advantage.} GRPO~\citep{shao2024deepseekmathpushinglimitsmathematical} computes advantages via group-wise normalization. For each prompt, it samples $K$ candidates $\{\bx^{k}\}_{k=1}^{K}$ and normalizes rewards within the group:
\begin{equation}
\label{eq:grpo_adv}
\mathcal{A}(\bx^{i},c)
= \frac{\mathcal{R}(\bx^{i},c) - \mathrm{mean}\{\mathcal{R}(\bx^{k},c)\}}
{\mathrm{std}\{\mathcal{R}(\bx^{k},c)\}}.
\end{equation}
\textbf{Policy Density in Flow Models.} For flow models, the policy in Eq.~\eqref{eq:pg} is defined by the reverse-time SDE transition~\citep{fan2023dpok,liu2025flow}. Under the flow parameterization, each step follows a Gaussian transition
\begin{equation}
\label{eq:x_tdt_rewrite}
\bx_{t-1} = \mu_\theta(\bx_t, t, c) + \sigma_t \sqrt{\mathrm{d}t}\, \bm{\epsilon}, \bm{\epsilon}\sim\mathcal{N}(\mathbf{0},\mathbf{I}) \quad  \Rightarrow \quad  p_\theta(\bx_{t-1}\mid \bx_t,c)=\mathcal{N}\bigl(\bx_{t-1};\mu_\theta(\bx_t,t,c),\sigma_t^2\mathrm{d}t\,\mathbf{I}\bigr),
\end{equation}
where covariance $\sigma_t^2\mathrm{d}t\,\mathbf{I}$ is fixed by noise schedule, while mean $\mu_\theta$ is given by the velocity field\citep{liu2025flow}:
\begin{equation}
\label{eq:mu_theta_rewrite}
\mu_\theta(\bx_t, t, c) = \bx_t + \Bigl[ v_\theta(\bx_t, t, c) + \frac{\sigma_t^2}{2t}\bigl(\bx_t + (1 - t)v_\theta(\bx_t, t, c)\bigr)\Bigr]\mathrm{d}t.
\end{equation}

\textbf{Flow-Based GRPO Objective.} Flow-GRPO optimizes the policy gradient in Eq.~\eqref{eq:pg} with the group-relative advantage~in Eq.~\eqref{eq:grpo_adv}, using a clipped proximal surrogate:
\begin{equation}
    \label{eq:grpo_objective}
    \mathcal{J}_{\text{GRPO}}(\theta)
    = \mathbb{E}_{i,t}\big[\min\big(\rho_{t,i}(\theta)\mathcal{A}_{i},\;
    \operatorname{clip}(\rho_{t,i}(\theta), 1-\epsilon, 1+\epsilon)\mathcal{A}_{i}\big)\big].
    \end{equation}
where $\rho_{t,i}(\theta)=\frac{p_{\theta}(\bx_{t-1}^{i}\mid \bx_t^{i},c)}{p_{\theta_{\text{old}}}(\bx_{t-1}^{i}\mid \bx_t^{i},c)}$ is the probability ratio between the current and old policy.

\textbf{Entropy Mechanism in LLM RL.} Token-level policy entropy from $p_\theta(\bx_t\mid \bx_{<t})$ is widely used for entropy regularization to encourage exploration~\citep{ziebart2008maximum,haarnoja2017reinforcement,haarnoja2018soft,zhu2506surprising,jiang2025rethinking,shen2025exploring}. In this context, \citet{cui2025entropy} establish an empirical relationship between policy entropy and validation reward:
\begin{equation}
    \label{eq:fit}
    \mathcal{R}=-a\exp(\mathcal{H})+b,
    \end{equation}
    where $\mathcal{R}$ denotes the validation reward and $a,b>0$ are coefficients.

%% file: sections/3_motivation.tex
\section{Understanding the Failure of Policy Entropy Constraints}
\label{sec:motivation}

Entropy regularization is a standard mechanism for mitigating diversity collapse in RL. In flow-based RL, however, a direct entropy constraint on the generated-sample distribution is difficult to instantiate, motivating a closer examination of tractable entropy surrogates. We identify a paradox: \textit{the policy entropy $\mathcal{H}(p_\theta(\bx_{t-1}\mid \bx_t,c))$ remains invariant during training, while diversity collapse still occurs}. This section analyzes this phenomenon through two questions: (1) why policy entropy remains constant during training, unlike in LLMs; and (2) why collapse still occurs despite this invariance.

\subsection{From Marginal Entropy to Policy Entropy}
\label{subsec:intractable-entropy}

\textbf{Challenge in Estimating Marginal Entropy.} In reinforcement learning, entropy regularization is typically applied to the marginal entropy $\mathcal{H}(p_\theta(\mathbf{x}_0 \mid c))$. In contrast, flow and diffusion models are trained through conditional transition distributions $p_\theta(\mathbf{x}_{t-1} \mid \mathbf{x}_t, c)$. As a result, evaluating $\mathcal{H}(p_\theta(\mathbf{x}_0 \mid c))$ requires access to the exact log-density of the marginal $p_\theta(\mathbf{x}_0 \mid c)$, which is widely recognized as an intractable issue in diffusion and flow model~\citep{song2021scorebasedsde,grathwohl2019ffjord,wang2022diffusion,desanti2025provable,desanti2025fdc,celik2025diffmaxent,maduabuchi2026ecfm}.

\textbf{Joint Entropy \& Policy Entropy.} A natural alternative is to consider the entropy of the full trajectory $\bx_{0:T}$ rather than the final sample $\bx_{0}$. In RL for LLMs~\citep{zhu2506surprising, jiang2025rethinking, shen2025exploring, cheng2025reasoning,liu2025prorl}, the token-level policy entropy of $p_\theta(\cdot\mid \bx_{<t})$ is widely used for regularization. Similarly, we consider using the policy entropy of $p_\theta(\bx_{t-1}\mid\bx_t,c)$ in flow models to obtain the joint entropy, as follows:
\begin{equation}
p_\theta(\bx_{0:T}|c)=p(\bx_T|c)\prod_{t=1}^{T}p_\theta(\bx_{t-1}|\bx_t,c)=p(\bx_T)\prod_{t=1}^{T}p_\theta(\bx_{t-1}|\bx_t,c),\quad(\bx_T\!\perp\!c).
\end{equation}
Considering that the initial distribution $\bx_{T}$ follows standard Gaussian distribution $\mathcal{N}(\mathbf{0}, \mathbb{I})$, the joint entropy $\mathcal{H}(p_\theta(\bx_{0:T}\mid c))$ is linearly related to the policy entropy $\mathcal{H}(p_\theta(\bx_{t-1}\mid\bx_t,c))$, as follow:
\begin{equation*}
\mathcal{H}(p_\theta(\bx_{0:T}|c))
=-\mathbb{E}_{p_\theta(\bx_{0:T}|c)}[\log p(\bx_T)+\sum_{t=1}^{T}\log p_\theta(\bx_{t-1}|\bx_t,c)]
=\mathcal{H}(p(\bx_T))+\sum_{t=1}^{T}\mathcal{H}(p_\theta(\bx_{t-1}|\bx_t,c)).
\end{equation*}
where $\mathcal{H}(p(\bx_T))$ is constant because $\bx_T$ is drawn from a fixed Gaussian prior.

Thus, by analogy with token-level entropy in LLMs, the policy entropy appears to be a natural tractable surrogate for trajectory-level entropy in flow-based RL. However, we observe that it is fixed by during RL fine-tuning and therefore cannot reflect the diversity collapse observed. Below, we provide the theoretical analysis in Sec.~\ref{subsec:constant-entropy} and verify this phenomenon empirically in Fig.~\ref{fig:motivation}.

\subsection{Constant Policy Entropy in Flow-based RL}
\label{subsec:constant-entropy}
We now formalize the per-step policy entropy induced by the reverse Gaussian transition and show that it is determined solely by the prescribed noise schedule.

\begin{definition}[Policy Entropy in Flow-based RL]
\label{def:entropy-in-flow-matching-rl}
For a prompt $c\in\mathcal{D}$ and timestep $t$, the per-step conditional entropy is defined as:
\begin{equation}
\label{eq:entropy-in-flow-matching-rl}
\mathcal{H}_t(p_\theta) = -\mathbb{E}_{\{\mathbf{x}_{\tau}\}_{\tau=0}^{T}\sim p_\theta(\cdot\mid c)}\!\left[\log p_\theta(\mathbf{x}_{t-1}\mid\mathbf{x}_t,c)\right],
\end{equation}
where the policy $p_\theta(\mathbf{x}_{t-1}\mid \mathbf{x}_t, c)$ is given by the Gaussian transition induced by the flow model $\theta$, i.e.,
$\mathcal{N}(\mathbf{x}_{t-1}; \mu_\theta(\mathbf{x}_t,t,c), \sigma_t^2\,\mathrm{d}t\,\mathbf{I})$, defined as follows:
\begin{equation}\label{eq:xxxxxlog_prob}
\log p_\theta(\mathbf{x}_{t-1}\mid\mathbf{x}_t,c) = -\frac{\|\mathbf{x}_{t-1}-\mu_\theta(\mathbf{x}_t,t,c)\|^2}{2\sigma_t^2\,\mathrm{d}t} + C_t,\qquad C_t = -\tfrac{d}{2}\log\!\big(2\pi\sigma_t^2\,\mathrm{d}t\big),
\end{equation}
where $d$ denotes the dimensionality of $\mathbf{x}_t$, and $C_t$ is constant term that is independent of $\theta$.
\end{definition}

Definition~\ref{def:entropy-in-flow-matching-rl} provides the closed-form density of the policy probability and its entropy. To further analyze the quantitative property of policy entropy, we provide the following property.

\begin{property}[Constant Policy Entropy in Flow-based RL]
\label{prop:constant-entropy}
For the reverse-time transition in Eq.~\eqref{eq:x_tdt_rewrite}, the per-step conditional entropy at timestep $t$ satisfies (see in App.~\ref{app:proof1}):
\begin{equation}
\mathcal{H}_t(p_\theta) = \frac{d}{2} - C_t.
\end{equation}
\end{property}
Property~\ref{prop:constant-entropy} shows that the policy entropy is invariant to the model parameters $\theta$ and depends only on the noise schedule and dimensionality. Intuitively, variations in $\theta$ only affect the Gaussian mean $\mu_\theta(\mathbf{x}_t,t,c)$ while leaving the covariance $\sigma_t^2\,\mathrm{d}t\,\mathbf{I}$ unchanged. Thus, the policy entropy remains constant during model $\theta$ fine-tuning, since the covariance is fixed by the noise schedule.

\begin{wrapfigure}{r}{0.56\textwidth}
\centering
\vspace{-1.3em}
\includegraphics[width=\linewidth]{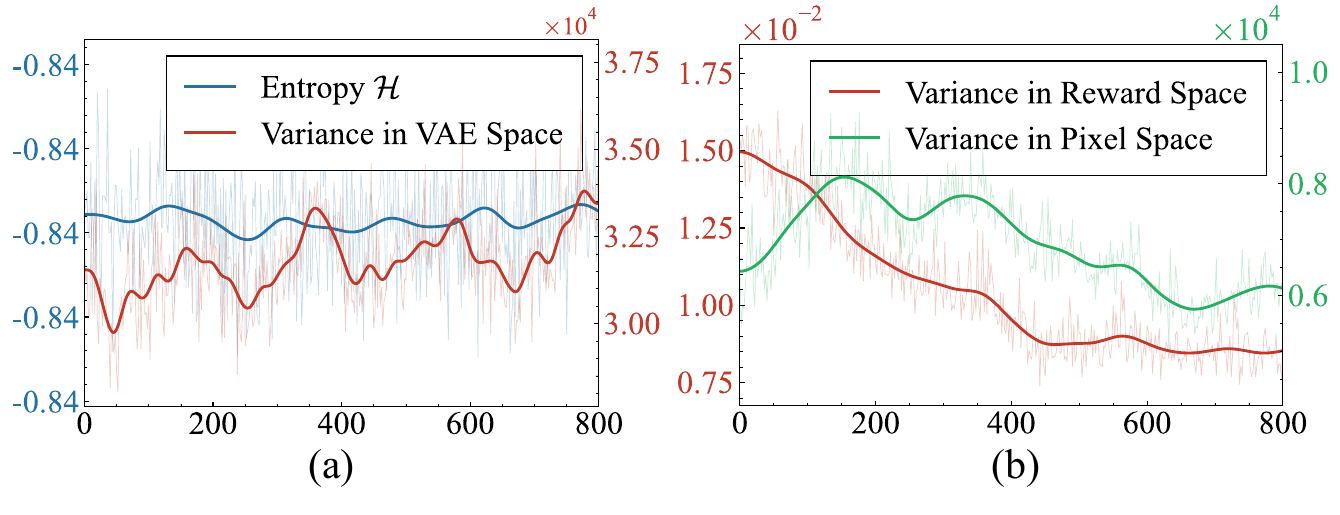}
\vspace{-2.0em}
\caption{\textbf{Training-time Metrics.} (a) Entropy vs. variance in VAE space.  (b) Variance in pixel and reward space.}
\label{fig:motivation}
\vspace{-2.0em}
\end{wrapfigure}

Beyond the theoretical analysis, we empirically verify the constant policy entropy property in Fig.~\ref{fig:motivation}. As shown in Fig.~\ref{fig:motivation}(a), during training, the entropy and variance in the VAE space remain nearly invariant. \textbf{\textit{However, diversity in the reward-perceived space still decreases over training.}} Consistent trends are also observed in the pixel and reward spaces, as shown in Fig.~\ref{fig:motivation}(b).  Next, we analyze the underlying reason for this apparent contradiction.

\subsection{Diversity Collapse behind Constant Policy Entropy}
\label{subsec:diversity-collapse}

The previous section explains why entropy stays constant. We next examine why diversity still collapses. Starting from the policy-gradient objective in Eq.~\eqref{eq:pg}, the following corollary characterizes the local policy-gradient direction underlying GRPO, with proof deferred to App.~\ref{app:proof-unified-grpo}.

\begin{corollary}[Mode-Seeking Optimization under On-Policy Policy Gradient]
\label{cor:unified-grpo-main}
For policy gradient methods~\citep{williams1992simple} in Eq.~\eqref{eq:pg}, with $p_{\theta_{\mathrm{old}}}$ and $\mathcal{A}$ fixed within an update, the gradient satisfies:
\begin{equation*}
\nabla_\theta \mathcal{J}(\theta) = \nabla_\theta \Big( -D_{\mathrm{KL}}\big(p_\theta(\bx_{0:T}|c) \,\|\, p_{\mathcal{A}}(\bx_{0:T}|c)\big) + D_{\mathrm{KL}}\big(p_\theta(\bx_{0:T}|c) \,\|\, p_{\theta_{\mathrm{old}}}(\bx_{0:T}|c)\big) \Big),
\end{equation*}
where $p_\theta(\mathbf{x}_{0:T}|c)$ is the joint distribution, and $p_{\mathcal{A}}(\bx_{0:T}|c)$ is the target distribution induced by advantage on the same trajectory space $p_{\mathcal{A}}(\bx_{0:T}|c)=Z_{\mathcal{A}}(c)^{-1}p_{\theta_{\mathrm{old}}}(\bx_{0:T}|c)\exp(\mathcal{A}(\bx_0,c))$ with $Z_{\mathcal{A}}(c)=\int p_{\theta_{\mathrm{old}}}(\bx_{0:T}|c)\exp(\mathcal{A}(\bx_0,c))\,\mathrm{d}\bx_{0:T}$.
\end{corollary}

Corollary~\ref{cor:unified-grpo-main} shows that the unclipped update direction locally minimizes the KL divergence to the advantage-induced target $p_{\mathcal{A}}$ while the KL term w.r.t. $p_{\theta_{\mathrm{old}}}$ acts as a regularizer, with GRPO applying trust-region clipping to this update. Since both KL terms place $p_\theta$ as the first argument, they correspond to the reverse KL, which is mode-seeking~\citep{ji2024towards} and drives the policy toward high-reward regions in the perceptual space, while being regularized toward the old policy in the generation space.

\textbf{Conceptual Illustrations.} Intuitively, the target distribution $p_{\mathcal{A}}$ is sampled from the old policy $\theta_\text{old}$, and thus unexplored perceptual modes are excluded from optimization. As a result, the policy model is encouraged toward a single high-reward mode with limited coverage, leading to diversity collapse. Repeated on-policy updates therefore converge probability mass on a narrow high-reward region in the perceptual space, as shown in Fig.~\ref{fig:mode_mean_seeking}. In contrast, an ideal model augmented with a perceptual signal should promote coverage over multiple high-reward regions. Empirically, Fig.~\ref{fig:reward-variance-training} shows that the dispersion of $p_{\mathcal{A}}$ decreases during training, with both reward variance and sample diversity declining, consistent with the contraction predicted by Corollary~\ref{cor:unified-grpo-main}.

\begin{figure*}[t]
    \centering
    \begin{minipage}[t]{0.49\textwidth}
        \vspace{0pt}
        \centering
        \includegraphics[width=\linewidth]{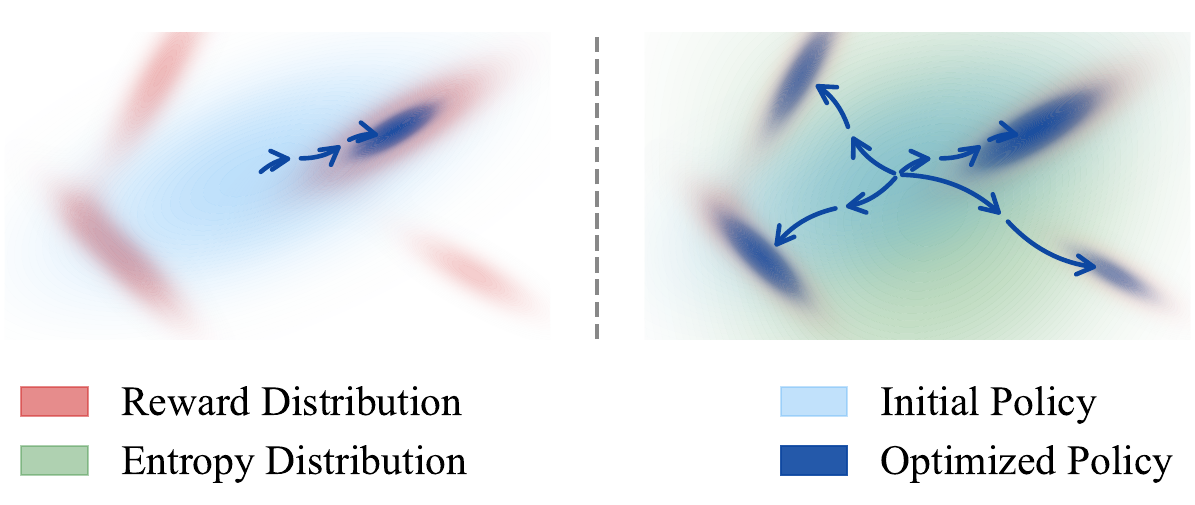}
        \caption{\textbf{Optimization Behavior.} GRPO shows mode-seeking behavior and converges to a single-peak distribution, while an ideal model should encourage coverage of multiple regions.}
        \label{fig:mode_mean_seeking}
    \end{minipage}\hfill
    \begin{minipage}[t]{0.49\textwidth}
        \vspace{0pt}
        \centering
        \includegraphics[width=\linewidth]{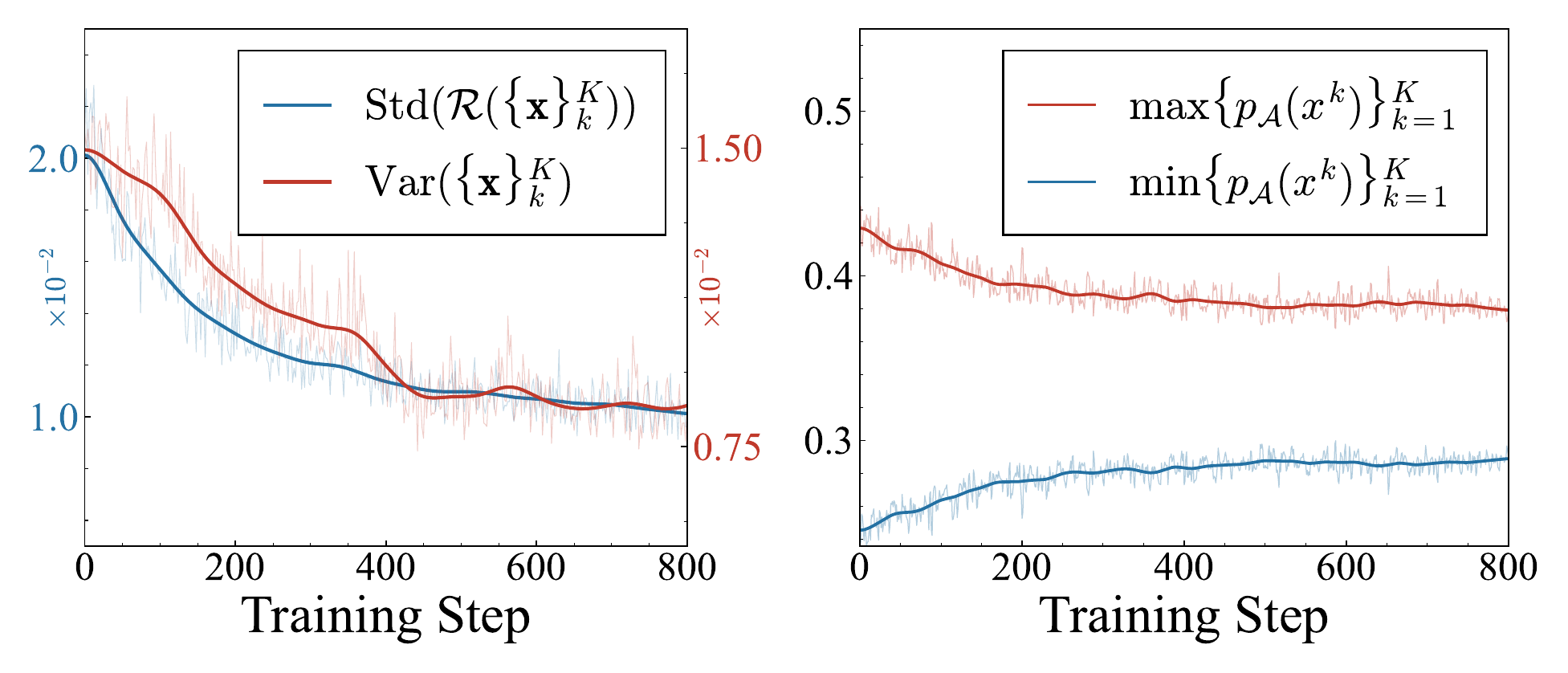}
        \vspace{-1.2em}
        \caption{\textbf{Training Dynamics Metrics.} (a) standard deviation of reward and variance of samples; (b) maximum and minimum advantage probabilities $p_{\mathcal{A}}$ during training.}
        \label{fig:reward-variance-training}
    \end{minipage}
    \vspace{-0.5em}
\end{figure*}

\textbf{Summary.} The policy entropy in the generation space remains invariant, as its covariance is fixed by the noise schedule; however, mode-seeking optimization drives the model to concentrate on a single high-reward mode with low perceptual diversity.

%% file: sections/4_method.tex
\section{Method}
\label{sec:method}

In this section, we address the issue of diversity collapse. Intuitively, diversity collapse can be seen as the feature variance reduction in human perception. Our analysis above reveals a key factor: \textit{\textbf{the standard policy entropy in flow models is unable to capture this collapse}}, undermining the effectiveness of existing entropy-based regularization for flow matching. This motivates us to define a new entropy concept, \textit{perceptual entropy}, which captures diversity collapse in perceptual space while remaining consistent with existing entropy-based regularization methods.

\begin{tcolorbox}[
    enhanced,
    frame hidden,
    colback=gray!7,
    borderline west={1.5pt}{0pt}{gray!55},
    sharp corners,
    boxsep=0pt,
    left=6pt, right=5pt, top=3.5pt, bottom=3.5pt,
    before skip=5pt, after skip=3pt
]
\textbf{Insight (Representation Perspective).} Standard VAEs primarily extract low-level features, making VAE-space entropy $\mathcal{H}$ insensitive to semantic diversity. For example, slightly altering eye size in a portrait shifts VAE latents while leaving high-level semantics nearly unchanged. In contrast, perceptual entropy $\mathcal{H}_{\mathrm{perc}}$ in a high-level representation space directly captures semantic variations and reliably detects diversity collapse.
\end{tcolorbox}

\subsection{Perceptual Entropy}
\label{subsec:perceptual-entropy}

To capture sample diversity via policy entropy, we introduce the concept of \textit{perceptual entropy}. It preserves the structure of standard policy entropy but is defined in the perceptual space rather than the generative latent space, as follows.

\begin{definition}[Perceptual Entropy]
\label{def:perceptual-entropy}
Let $\phi = \mathcal{D}_{\text{vae}} \circ \mathcal{E}_{p}$ be a mapping where $\mathcal{E}_{p}$ is a perceptual encoder and $\mathcal{D}_{\text{vae}}$ is a VAE decoder. Given $k$ rollout noised samples from the exploration process Eq.~\eqref{eq:x_tdt_rewrite},  the perceptual probability is defined as:
\begin{equation}
p_\theta^{\mathrm{perc}}(\mathbf{x}_{t-1}^k\mid\mathbf{x}_t^k,c)
:=\mathcal{N}\!\left(\phi(\mathbf{x}_{t-1}^k);\,\phi(\mathbf{m}^k),\,\sigma_t^2\,\mathrm{d}t\,\mathbf{I}\right).
\end{equation}
where $\mathbf{m}^k$ denotes the old-policy mean $\mu_{\theta_{\mathrm{old}}}(\mathbf{x}_t^k,t,c)$, and the perceptual entropy is then:
\begin{equation} \label{eq:perceptualentropy}
\mathcal{H}_{\text{perc}}(p_\theta) =
\mathbb{E}_{t,k}\!\left[-\log p_\theta^{\mathrm{perc}}(\mathbf{x}_{t-1}^k\mid\mathbf{x}_t^k,c)\right].
\end{equation}
\end{definition}

Definition~\ref{def:perceptual-entropy} introduces an entropy notion that quantifies diversity collapse in the perceptual space. Specifically, noised samples $\mathbf{x}_t$ and the old-policy mean $\mathbf{m}^k$ in the flow model are first decoded into the pixel space by a VAE decoder $\mathcal{D}_{\text{vae}}$, and then mapped into the perceptual space by a perceptual encoder $\mathcal{E}_{p}$. In this space, perceptual likelihoods $p_\theta^{\mathrm{perc}}(\mathbf{x}_{t-1}^k \mid \mathbf{x}_t^k, c)$ are computed under a local Gaussian centered at $\phi(\mathbf{m}^k)$, and used to define the perceptual entropy in Eq.~\eqref{eq:perceptualentropy}.

\textbf{Perceptual Encoder.} Here, we discuss the choice of perceptual encoder. For neural network-based rewards, we use the PickScore encoder to avoid additional computational cost. As discussed in Sec.~\ref{sec:generalizability} and Tab.~\ref{tab:runtime_overhead}, this additional perceptual encoder is both effective and efficient, yielding consistent gains with minimal overhead.

\textbf{Surrogate Scoring without Jacobians.} Definition~\ref{def:perceptual-entropy} formulates $p_\theta^{\mathrm{perc}}$ directly as a feature-space scoring distribution without the spatial-transformation Jacobian, analogous to score distillation methods that optimize 2D priors directly in feature space~\citep{poole2022dreamfusion,wang2023prolificdreamer,lin2023magic3d,wang2024animatabledreamer}. Following on-policy REINFORCE formulations~\citep{ahmadian2024back,hu2025reinforcepp,liu2025understanding}, perceptual scores are evaluated on fixed old-policy rollouts as shaped rewards, while the GRPO likelihood ratio performs the policy update.

To connect perceptual entropy to sample diversity, we analyze its relationship to the conditional feature-space variance under a local mean-preserving approximation.

\begin{remark}[Connection to Conditional Feature-Space Variance]
\label{prop:perceptual-entropy-variance}
Let $\mathbf{z}_{t-1}=\phi(\mathbf{x}_{t-1})$. Assume that, within a local reverse step, the perceptual encoder is approximately mean-preserving, i.e., $\mathbb{E}\!\left[\mathbf{z}_{t-1}\mid \mathbf{x}_t\right]\approx \phi(\mathbf{m})$. Then perceptual entropy tracks the expected conditional feature-space variance:
\begin{equation}
\mathcal{H}_{\text{perc}}(p_\theta) \approx
\mathbb{E}_{t,\mathbf{x}_t}\!\left[
\frac{\mathrm{Var}(\mathbf{z}_{t-1}\mid \mathbf{x}_t)}
{2\sigma_t^2 \mathrm{d}t}\right] - C_t.
\end{equation}
\end{remark}

Remark~\ref{prop:perceptual-entropy-variance} shows that perceptual entropy captures local perceptual dispersion, paralleling the latent-space variance relation in App.~\ref{app:proof2} and aligning with entropy-regularization analyses in LLMs~\cite{cui2025entropy,jiang2025rethinking}.

\textbf{Empirical Entropy Mechanism.} To better understand this mechanism, we empirically study the relationship between perceptual entropy and reward in flow models, in analogy to the empirical finding in LLMs~\cite{cui2025entropy}. We therefore propose a similar empirical entropy mechanism for flow models:
\begin{equation}\label{eq:entropy_mechanism_flow}
    \mathcal{R} = -a\exp(\mathcal{H}_{\text{perc}}) + b.
\end{equation}
where $a$ and $b$ are empirical coefficients.

\begin{wrapfigure}[13]{r}{0.65\textwidth}
\centering
\vspace{-1.5em}
\includegraphics[width=\linewidth]{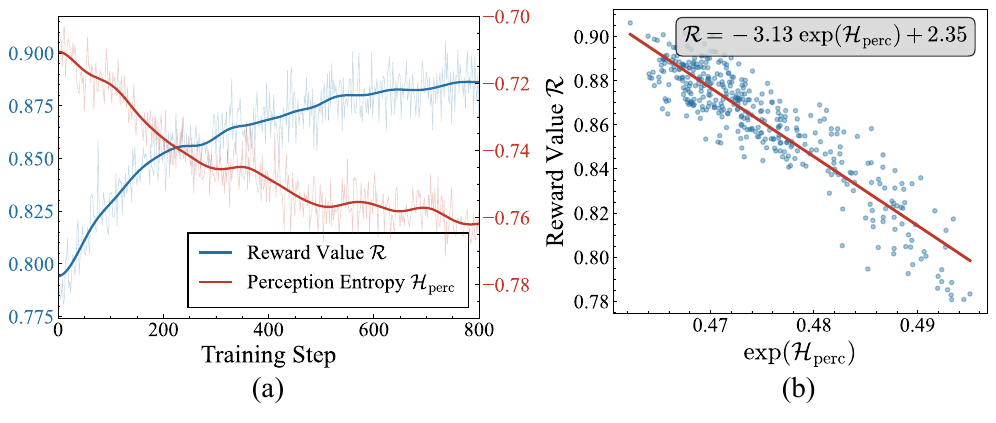}
\vspace{-2.0em}
\caption{\textbf{Relationship between Reward and Perceptual Entropy.}(a) Reward $\mathcal{R}$ vs. perceptual entropy $\mathcal{H}_{\text{perc}}$. (b) Reward $\mathcal{R}$ vs. perceptual entropy $\mathcal{H}_{\text{perc}}$.}
\label{fig:empirical_entropy_validation}
\end{wrapfigure}

As shown in Fig.~\ref{fig:empirical_entropy_validation}, we report the relationship between reward $\mathcal{R}$ and perceptual entropy $\mathcal{H}_\text{perc}$. These results along with the theoretical analyses  demonstrate that perceptual entropy provides \textbf{\textit{an effective identifier of diversity}}, \textbf{\textit{aligning with existing research}} for LLMs~\cite{cui2025entropy}.

This naturally motivates regularizing perceptual entropy to mitigate diversity collapse. Here, we introduce two approaches based on perceptual entropy: (1) \textit{Perceptual Entropy Constraint} (\textit{Perc-EC}), which treats perceptual entropy as a principled entropy identifier and incorporates it into a standard entropy-regularization framework; and (2) \textit{Perceptual VAE Constraint} (\textit{Perc-VC}), which leverages perceptual entropy to encourage consistency between the original and collapsed perceptual spaces.

\subsection{Perceptual-Regularized Strategies for GRPO}
\label{sec:perceptual-shaping}

\textbf{Objective.} Standard RLHF inherently drives the policy toward a narrow subset of high-reward samples, thereby inducing diversity collapse. To mitigate this issue, we augment the RLHF objective in Eq.~\eqref{eq:rlhf} with a perceptual regularization term $\Omega(p_\theta^\text{perc})$, as follows:
\begin{equation}\label{eq:general-perceptual-objective}
\max_{\theta}\ \mathcal{J}(\theta)
=\mathbb{E}_{\mathbf{x}\sim p_\theta}\bigl[\mathcal{R}(\mathbf{x}, c)\bigr]
+\lambda\,\Omega(p_\theta^\text{perc}).
\end{equation}
Here, $\Omega(p_\theta^\text{perc})$ is a perceptual regularizer that encourages exploration of a broader perceptual space, while $\mathbb{E}_{\mathbf{x}\sim p_\theta}\bigl[\mathcal{R}(\mathbf{x}, c)\bigr]$ promotes high-preference samples through reward maximization.

\noindent\textbf{Perceptual Entropy Constraint (\textit{Perc-EC}).}\label{sec:pec}  Since perceptual entropy serves as a reliable surrogate for conditional variance as shown in Remark~\ref{prop:perceptual-entropy-variance}, a natural choice is to maximize it directly, as follows:
\begin{equation}\label{eq:empirical-entropy-constraint}
\max_{\theta}\ \mathcal{J}_{\textit{Perc-EC}}(\theta)
=\mathbb{E}_{\mathbf{x}\sim p_\theta}\bigl[\mathcal{R}(\mathbf{x},c)\bigr]
+\lambda\,\mathcal{H}_{\text{perc}}(p_\theta).
\end{equation}
Here, we set $\Omega=\mathcal{H}_{\text{perc}}(p_\theta)$ to encourage the policy to explore a broader perceptual space.

\noindent\textbf{Perceptual VAE Constraint (\textit{Perc-VC}).}\label{sec:pec-vae}
While \textit{Perc-EC} directly promotes variance, an alternative is to anchor the policy distribution to the diversity-preserving generation-space prior by setting $\Omega=-D_{\mathrm{KL}}(p_\theta^{\mathrm{perc}}\Vert p_\theta^{\mathrm{vae}})$, where $p_\theta^{\mathrm{vae}}$ is the policy $p_\theta$ in Eq.~\eqref{eq:x_tdt_rewrite}, as follows:
\begin{equation}\label{eq:kl-constraint}
\max_{\theta}\ \mathcal{J}_{\textit{Perc-VC}}(\theta)
=\mathbb{E}_{\mathbf{x}\sim p_\theta}\bigl[\mathcal{R}(\mathbf{x},c)\bigr]
-\lambda\,D_{\mathrm{KL}}\!\left(p_\theta^{\mathrm{perc}}\,\Vert\,p_\theta^{\mathrm{vae}}\right).
\end{equation}

\textbf{Perceptual Reward Shaping.} Although theoretically sound, optimizing \textit{Perc-EC} and \textit{Perc-VC} as independent auxiliary losses introduces empirical challenges. Direct entropy maximization injects unstructured stochasticity~\cite{vanlioglu2025entropy}, leading to competing gradients against the primary reward and instability training,  as shown in the red line in Fig.~\ref{fig:kl_cmp}(a). To address this, we incorporate perceptual constraints directly into the reward signal, a strategy we term \textit{perceptual reward shaping}. The shaped rewards are computed by aggregating instantaneous perceptual quantities across reverse timesteps:
\begin{equation}\label{eq:reward_redefine_main}
\begin{aligned}
\mathcal{J}_{\textit{Perc-EC}}(\theta) = \mathbb{E}_{c,k}\!\Big[\underbrace{\mathcal{R}(\mathbf{x}_0^k, c) - \lambda\,\mathbb{E}_t\!\left[\log p_\theta^{\mathrm{perc}}(\mathbf{x}_{t-1}^k\mid\mathbf{x}_t^k,c)\right]}_{\text{Shaped Reward } \tilde{\mathcal{R}}_{\textit{Perc-EC}}^k}\Big].
\end{aligned}
\end{equation}
This incorporates perceptual constraints directly into the reward signal, ensuring that perceptual and task rewards are jointly normalized within the same GRPO rollout group, as shown by the blue line in Fig.~\ref{fig:kl_cmp}(a). Similarly, the shaped reward for the \textit{Perc-VC} constraint is given by:
\begin{equation}\label{eq:shaped-rewards}
\tilde{\mathcal{R}}_{\textit{Perc-VC}}^k = \mathcal{R}(\mathbf{x}_0^k,c) - \lambda\,\mathbb{E}_t\!\left[\log p_\theta^{\mathrm{perc}}(\mathbf{x}_{t-1}^k\mid\mathbf{x}_t^k,c) - \log p_\theta^{\mathrm{vae}}(\mathbf{x}_{t-1}^k\mid\mathbf{x}_t^k,c)\right].
\end{equation}

To optimize the shaped rewards, we adopt the standard GRPO clipped surrogate, as follows:
\begin{equation}\label{eq:grpo_perc}
\mathcal{J}_{m\text{-GRPO}}(\theta)
= \mathbb{E}_{c,\,\{\mathbf{x}_t^k\}_{t,k}}\!\Big[\min\big(\rho_{t,k}(\theta)\,\hat{\mathcal{A}}_{k}^{m},\;
\operatorname{clip}(\rho_{t,k}(\theta), 1\!-\!\epsilon, 1\!+\!\epsilon)\,\hat{\mathcal{A}}_{k}^{m}\big)\Big].
\end{equation}
where $\hat{\mathcal{A}}_{k}^{m}$ is the advantage in Eq.\eqref{eq:grpo_adv}, and $m\in\{\textit{Perc-EC},\textit{Perc-VC}\}$ denotes the constraint type.

\begin{wrapfigure}[10]{r}{0.65\textwidth}
    \centering
    \vspace{-1.4em}
    \includegraphics[width=\linewidth]{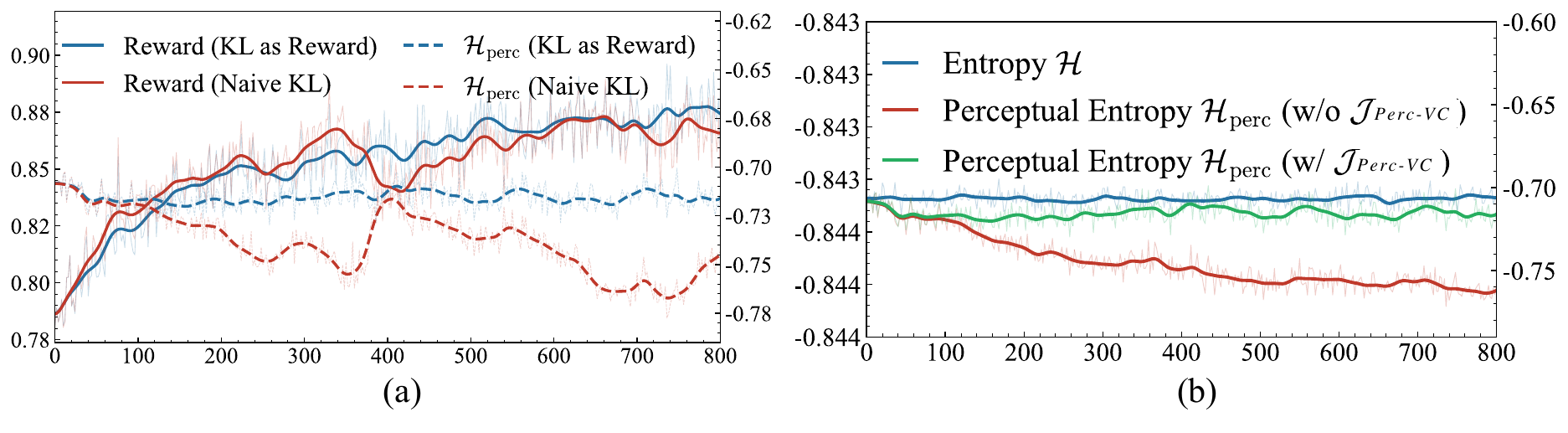}
    \vspace{-1.8em}
    \caption{\textbf{Optimization Comparison.} (a) Native KL (red) versus KL-as-reward (blue), showing reward and entropy dynamics; (b) $\mathcal{H}$ (blue) and $\mathcal{H}_{\mathrm{perc}}$ with (green) and without (red) $\mathcal{J}_{\textit{Perc-VC}}$.}
    \vspace{-1.2em}
    \label{fig:kl_cmp}
\end{wrapfigure}

As illustrated in Fig.\ref{fig:kl_cmp}, our method effectively mitigates these optimization issues. Compared to Native KL, which introduces competing gradients, perceptual reward shaping aligns perceptual constraints with the primary objective and stabilizes training dynamics, as shown by the blue line in Fig.\ref{fig:kl_cmp}(a). Moreover, incorporating the \textit{Perc-VC} constraint regularizes entropy maximization, leading to higher and more stable entropy and improved output diversity, as shown in Fig.~\ref{fig:kl_cmp}(b).

%% file: sections/5_experiments.tex
\section{Experiments}
\label{sec:experiments}

\subsection{Experimental Setup}

\textbf{Implementation.} We build upon Flow-GRPO~\yrcite{liu2025flow} and evaluate two base models, FLUX.dev~\yrcite{flux2024} and SD3.5-M~\yrcite{sd3}. For FLUX.dev, we use PickScore~\yrcite{kirstain2023pick} as the neural reward and compute perceptual entropy with PickScore and DINO~\yrcite{DBLP:journals/corr/abs-2104-14294} encoders. For SD3.5-M, we follow Flow-GRPO~\yrcite{liu2025flow}, which adopts rule-based rewards. FLUX.dev is trained on Pick-a-Pic~\yrcite{kirstain2023pick} with 37,523 prompts. For \textit{Perc-EC} and \textit{Perc-VC}, we tune $\lambda \in \{0.03, 0.05, 0.10\}$, with larger values to encourage diversity and smaller ones for quality. We also apply \textit{Perc-EC} and \textit{Perc-VC} to DDPO~\yrcite{black2023training} on FLUX.dev with PickScore, and train Wan2.1-T2V-1.3B~\yrcite{wan2025wan} for text-to-video generation with Flow-GRPO and VideoPickScore in UniRL~\yrcite{unirl_github}. Image generation models are trained on 8 GPUs, and video generation models are trained on 96 GPUs.

\textbf{Metrics.} Following MixGRPO~\cite{li2025mixgrpo}, we evaluate generation quality on the HPD test set using ImageReward~\yrcite{xu2023imagereward}, PickScore~\yrcite{kirstain2023pick}, Aesthetic Predictor v2.5~\yrcite{discus0434_aesthetic_2024}, CLIP~\yrcite{radford2021learning}, HPSv2.1~\yrcite{wu2023human}, and Unified Reward~\yrcite{wang2025unified}. For diversity, we sample 100 images per prompt across 85 PartiPrompts~\yrcite{yu2022scaling}, totaling 8,500 images per method, and report feature-space $L_2$ diversity in DINO and CLIP spaces alongside Vendi Scores~\yrcite{friedman2023vendi} in CLIP, DINO, ImageReward, and PickScore spaces. We also report an overall score by averaging min-max normalized quality and diversity metrics within each table.

\noindent
\textbf{Compared Methods.} We compare standard entropy regularization~\yrcite{schulman2017proximal}, Clip-Higher~\yrcite{yu2025dapo}, and covariance-aware Clip-Cov/KL-Cov~\yrcite{cui2025entropy}. Details are provided in App.~\ref{app:cov-flow}. We test both their original entropy terms and our perceptual entropy $\mathcal{H}_{\mathrm{perc}}$, then compare directly with \textit{Perc-VC} and \textit{Perc-EC}. Hyperparameters follow Flow-GRPO unless stated otherwise: entropy weight $0.05$, Clip-Cov/KL-Cov rate $0.25$, Clip-Higher threshold $0.1$, and KL weight $0.001$~\citep{liu2025flow}.

\subsection{Main Results}
\input{tables/main_results.tex}
\input{tables/sota.tex}

\begin{wrapfigure}{r}{0.56\textwidth}
    \centering
    \vspace{-1.3em}
    \includegraphics[width=\linewidth]{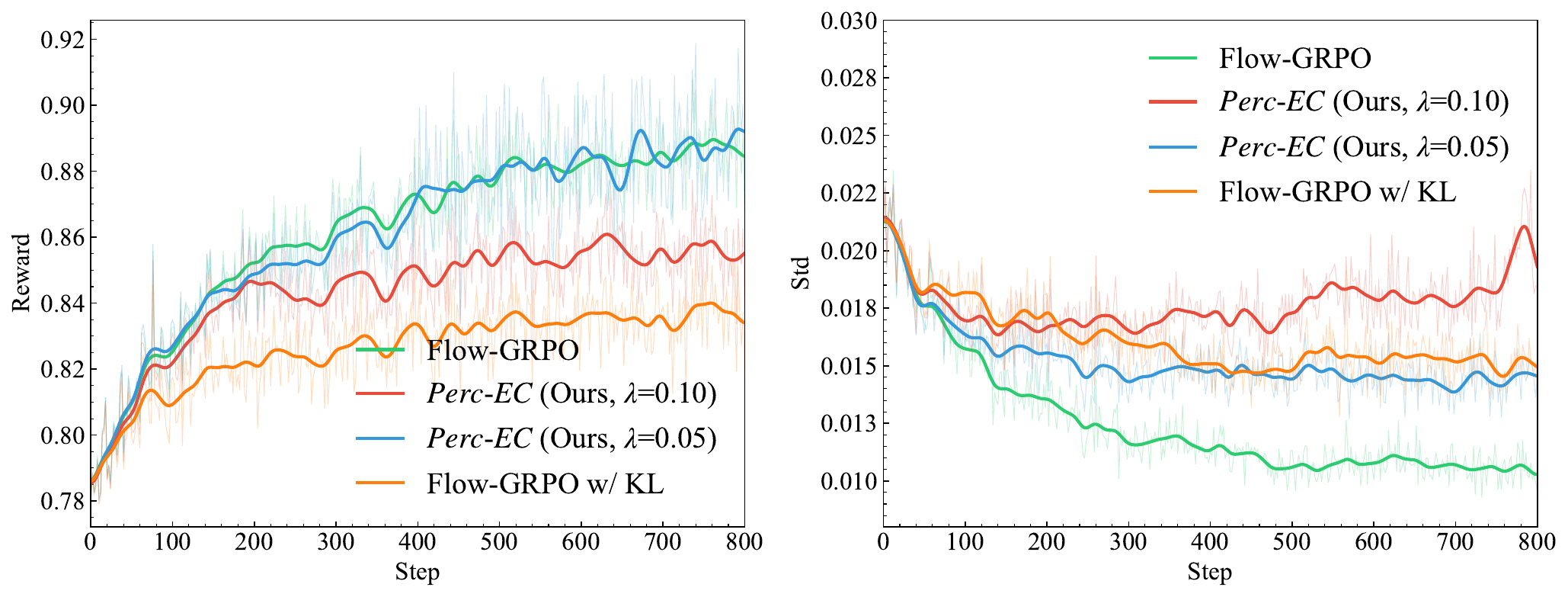}
    \caption{\textbf{Training Dynamics of Flow-GRPO, Flow-GRPO with KL, and \textit{Perc-EC} with Perceptual Entropy.} (a) Mean reward. (b) Reward standard deviation.}    \label{fig:exp_std_rw}
    \vspace{-2.0em}
\end{wrapfigure}
\textbf{Comparison with Entropy-Based Regularization.} As shown in Tab.~\ref{tab:main_results}, native per-step entropy $\mathcal{H}$ provides limited gains under flow matching, whereas perceptual entropy $\mathcal{H}_{\mathrm{perc}}$ is substantially more effective. Standard entropy regularization achieves an overall score of 0.245, comparable to Flow-GRPO with KL at 0.242. Replacing $\mathcal{H}$ with $\mathcal{H}_{\mathrm{perc}}$ raises the score to 0.427. \textit{Perc-EC} ($\lambda=0.10$, PickScore) achieves the best overall score of 0.755 and diversity average of 0.941.

\textbf{Comparison with SoTA Models.} As shown in Tab.~\ref{tab:sota_results}, existing text-to-image models exhibit a sharp quality--diversity trade-off. SDXL achieves the highest baseline diversity of 0.651 but a lower quality of 0.205, whereas Qwen-Image attains higher quality of 0.711 with substantially reduced diversity of 0.000. In contrast, \textit{Perc-EC} ($\lambda=0.10$, PickScore) reaches the best overall score of 0.751 with a diversity average of 0.872.

\textbf{Trade-Off between Diversity and Quality.} Fig.~\ref{fig:exp_std_rw} shows that Flow-GRPO improves quality but reduces diversity. KL alleviates this drop at a quality cost. \textit{Perc-EC} gives a better balance: $\lambda=0.10$ preserves diversity, while $\lambda=0.05$ keeps quality close to Flow-GRPO.

\textbf{Visualization \& User Study.} The teaser in Fig.~\ref{fig:teaser} and extended visualizations in Fig.~\ref{fig:app_vis} demonstrate that \textit{Perc-EC} preserves visual diversity across prompts and domains while maintaining high visual fidelity. User study details are provided in App.~\ref{app:user_study}.

\subsection{Ablation Study}

\begin{table*}[t]
\centering
\begin{minipage}[t]{0.48\textwidth}
\input{tables/ablation_perc_entropy.tex}
\end{minipage}
\hfill
\begin{minipage}[t]{0.48\textwidth}
\input{tables/runtime_overhead.tex}
\end{minipage}
\end{table*}

\textbf{Effectiveness of Perceptual Entropy.} Tab.~\ref{tab:main_results} reports KL-Cov at $0.403$ overall in VAE space and $0.452$ in PickScore space. Tab.~\ref{tab:ablation_perc_entropy} further isolates this effect in \textit{Perc-EC}, where replacing the native per-step entropy with $\mathcal{H}_{\mathrm{perc}}$ improves Aes.v2.5 from $6.372$ to $6.412$.

\textbf{Perceptual Reward Shaping vs. Native Regularization.} Fig.~\ref{fig:kl_cmp} shows that perceptual reward shaping is more stable than directly optimizing the native KL or entropy regularization. This stabilizer improves quality while preserving the diversity benefit from perceptual entropy.

\textbf{Perceptual Encoder.} \label{sec:generalizability} Our perceptual entropy can be computed not only with the reward model encoder, but also with a frozen DINO or CLIP encoder. As shown in Table~\ref{tab:main_results}, the performance gains generalize across different perceptual encoders. As shown in Table~\ref{tab:runtime_overhead}, the additional iteration-time overhead is marginal and peak memory usage is nearly unchanged, since the perceptual encoder is much smaller than the flow model and is evaluated asynchronously from generation and optimization.

\begin{figure}[h]
    \centering
    \begin{minipage}[t]{0.49\textwidth}
        \centering
        \includegraphics[width=\linewidth]{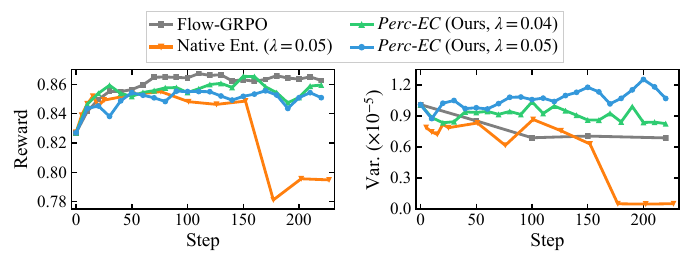}
        \captionsetup{labelfont=normalfont}
        \caption{\textbf{Text-to-Video Generalization.} Training dynamics on Wan2.1. \textbf{Left:} Reward $\mathcal{R}$. \textbf{Right:} Within-group perceptual variance.}
        \label{fig:video_eval}
    \end{minipage}
    \hfill
    \begin{minipage}[t]{0.49\textwidth}
        \centering
        \includegraphics[width=\linewidth]{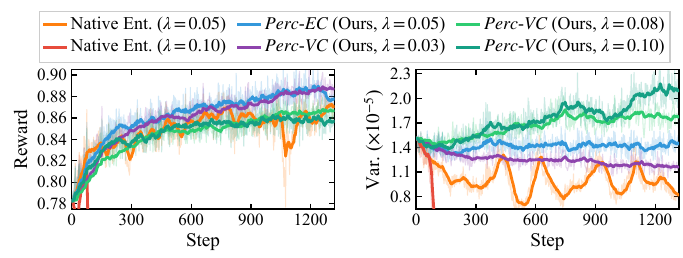}
        \captionsetup{labelfont=normalfont}
        \caption{\textbf{Generalization to DDPO.} Training dynamics under DDPO. \textbf{Left:} Reward $\mathcal{R}$. \textbf{Right:} Within-group perceptual variance.}
        \label{fig:ddpo_dynamics}
    \end{minipage}
\end{figure}

\textbf{Text-to-Video Generalization.} As shown in Fig.~\ref{fig:video_eval}, both Flow-GRPO and native entropy regularization erode within-group perceptual variance on Wan2.1 as reward rises, whereas \textit{Perc-EC} preserves perceptual diversity alongside high reward.

\textbf{Generalization to DDPO.} As shown in Fig.~\ref{fig:ddpo_dynamics}, native entropy regularization fails to prevent diversity collapse under DDPO: variance still declines at $\lambda=0.05$, while $\lambda=0.10$ quickly destabilizes training. In contrast, \textit{Perc-EC} maintains perceptual variance alongside high reward, and \textit{Perc-VC} enables a controllable quality--diversity trade-off that progressively expands perceptual variance.

\begin{figure}[t]
    \centering
    \includegraphics[width=\textwidth]{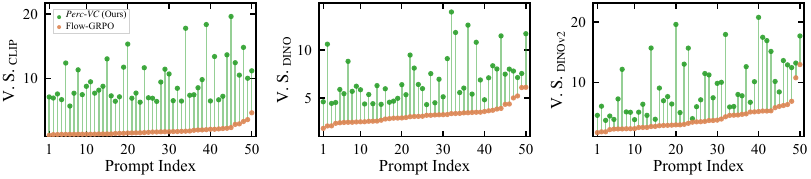}
    \captionsetup{labelfont=normalfont,textfont=normalfont}
    \caption{\textbf{Per-Prompt Vendi Scores.} For each feature space, the top-$50$ prompts is included, and each prompt is evaluated with $100$ samples. \textbf{Left:} CLIP. \textbf{Middle:} DINO. \textbf{Right:} DINOv2.}
    \label{fig:per_prompt_vs}
    \vspace{-1.0em}
\end{figure}

\textbf{Per-Prompt Diversity.} Fig.~\ref{fig:per_prompt_vs} reports the per-prompt Vendi Scores across the top-$50$ PartiPrompts with $100$ outputs per prompt in CLIP, DINO, and DINOv2 feature spaces.

%% file: tables/main_results.tex
\begin{table}[t]
\begingroup
\colorlet{blue}{black}
\centering
\caption{\textbf{Entropy Regularization Comparison.} Gray-shaded rows denote methods using perceptual entropy $\mathcal{H}_{\mathrm{perc}}$. Abbreviations: P.S.=PickScore, I.R.=ImageReward, U.R.=Unified Reward, and V.S.=Vendi Score.}
\vspace{0.4em}
\label{tab:main_results}
\resizebox{\textwidth}{!}{
\small
\setlength{\tabcolsep}{3.5pt}
\begin{tabular}{ll*{15}{c}}
\toprule
\multirow{2}{*}{Method} & \multirow{2}{*}{\shortstack[c]{Perceptual\\Space}} & \multicolumn{7}{c}{Quality$\uparrow$} & \multicolumn{7}{c}{Diversity$\uparrow$} & \multirow{2}{*}{Overall} \\
\cmidrule(lr){3-9} \cmidrule(lr){10-16}
  &   &  P.S.  &  Aes.v2.5  &  I.R.  &  CLIP  &  HPSv2.1  &  U.R.  &  Norm.Avg  & \textcolor{blue}{DINO} & \textcolor{blue}{CLIP} & $\mathrm{V.S.}_{\mathrm{DINO}}$ & $\mathrm{V.S.}_{\mathrm{CLIP}}$ & $\mathrm{V.S.}_{\mathrm{IR}}$ & $\mathrm{V.S.}_{\mathrm{PS}}$ & Norm.Avg &  \\
\midrule
\multicolumn{17}{c}{\textit{Base Model \& Baseline without Entropy Regularization}} \\
\midrule
Flux Dev & -- & 0.226 & 5.837 & 1.089 & 0.388 & 0.312 & 3.540 & \cellcolor{cvprblue!12}0.000 & \textcolor{blue}{0.70} & \textcolor{blue}{0.49} & \textcolor{blue}{5.14} & \textcolor{blue}{2.61} & \textcolor{blue}{1.26} & \textcolor{blue}{1.15} & \cellcolor{darkgreen!31}\textcolor{blue}{0.389} & \cellcolor{orange!22}\textcolor{blue}{0.195} \\
Flow-GRPO (\textit{Baseline}) & -- & 0.243 & 6.305 & 1.375 & 0.396 & 0.337 & 3.738 & \cellcolor{cvprblue!44}0.685 & \textcolor{blue}{0.54} & \textcolor{blue}{0.35} & \textcolor{blue}{2.82} & \textcolor{blue}{1.69} & \textcolor{blue}{1.13} & \textcolor{blue}{1.09} & \cellcolor{darkgreen!12}\textcolor{blue}{0.010} & \cellcolor{orange!29}\textcolor{blue}{0.347} \\
\midrule
\multicolumn{17}{c}{\textit{Existing Methods with Native Entropy $\mathcal{H}$ or Perceptual Entropy $\mathcal{H}_{\mathrm{perc}}$}} \\
\midrule
Flow-GRPO w/ KL & -- & 0.233 & 5.998 & 1.289 & 0.396 & 0.323 & 3.582 & \cellcolor{cvprblue!28}0.319 & \textcolor{blue}{0.61} & \textcolor{blue}{0.41} & \textcolor{blue}{3.64} & \textcolor{blue}{2.02} & \textcolor{blue}{1.18} & \textcolor{blue}{1.12} & \cellcolor{darkgreen!20}\textcolor{blue}{0.165} & \cellcolor{orange!24}\textcolor{blue}{0.242} \\
Clip-Higher & -- & 0.240 & 6.273 & 1.367 & 0.392 & 0.336 & 3.741 & \cellcolor{cvprblue!42}0.617 & \textcolor{blue}{0.55} & \textcolor{blue}{0.38} & \textcolor{blue}{3.01} & \textcolor{blue}{1.90} & \textcolor{blue}{1.17} & \textcolor{blue}{1.10} & \cellcolor{darkgreen!16}\textcolor{blue}{0.066} & \cellcolor{orange!29}\textcolor{blue}{0.342} \\
Entropy Reg. & VAE & 0.240 & 5.965 & 1.362 & 0.398 & 0.332 & 3.615 & \cellcolor{cvprblue!35}0.473 & \textcolor{blue}{0.53} & \textcolor{blue}{0.37} & \textcolor{blue}{2.74} & \textcolor{blue}{1.74} & \textcolor{blue}{1.15} & \textcolor{blue}{1.09} & \cellcolor{darkgreen!13}\textcolor{blue}{0.018} & \cellcolor{orange!24}\textcolor{blue}{0.245} \\
\rowcolor{gray!20} Entropy Reg. & PickScore & 0.239 & 6.079 & 1.365 & 0.396 & 0.332 & 3.604 & \cellcolor{cvprblue!35}0.478 & \textcolor{blue}{0.69} & \textcolor{blue}{0.51} & \textcolor{blue}{4.52} & \textcolor{blue}{2.58} & \textcolor{blue}{1.27} & \textcolor{blue}{1.16} & \cellcolor{darkgreen!30}\textcolor{blue}{0.376} & \cellcolor{orange!32}\textcolor{blue}{0.427} \\
Clip-Cov & VAE & 0.241 & 6.128 & 1.426 & 0.399 & 0.335 & 3.672 & \cellcolor{cvprblue!41}0.610 & \textcolor{blue}{0.59} & \textcolor{blue}{0.41} & \textcolor{blue}{3.42} & \textcolor{blue}{1.98} & \textcolor{blue}{1.15} & \textcolor{blue}{1.11} & \cellcolor{darkgreen!18}\textcolor{blue}{0.121} & \cellcolor{orange!30}\textcolor{blue}{0.366} \\
\rowcolor{gray!20} Clip-Cov & PickScore & 0.242 & 6.153 & 1.448 & 0.402 & 0.337 & 3.698 & \cellcolor{cvprblue!44}0.680 & \textcolor{blue}{0.60} & \textcolor{blue}{0.42} & \textcolor{blue}{3.49} & \textcolor{blue}{2.11} & \textcolor{blue}{1.19} & \textcolor{blue}{1.12} & \cellcolor{darkgreen!20}\textcolor{blue}{0.163} & \cellcolor{orange!32}\textcolor{blue}{0.422} \\
KL-Cov & VAE & 0.243 & 6.143 & 1.431 & 0.397 & 0.345 & 3.720 & \cellcolor{cvprblue!46}0.694 & \textcolor{blue}{0.59} & \textcolor{blue}{0.39} & \textcolor{blue}{3.47} & \textcolor{blue}{1.86} & \textcolor{blue}{1.15} & \textcolor{blue}{1.11} & \cellcolor{darkgreen!17}\textcolor{blue}{0.112} & \cellcolor{orange!31}\textcolor{blue}{0.403} \\
\rowcolor{gray!20} KL-Cov & PickScore & 0.243 & 6.296 & 1.474 & 0.398 & \underline{0.346} & 3.746 & \cellcolor{cvprblue!49}0.783 & \textcolor{blue}{0.59} & \textcolor{blue}{0.41} & \textcolor{blue}{3.42} & \textcolor{blue}{1.98} & \textcolor{blue}{1.15} & \textcolor{blue}{1.11} & \cellcolor{darkgreen!18}\textcolor{blue}{0.121} & \cellcolor{orange!34}\textcolor{blue}{0.452} \\
\midrule
\multicolumn{17}{c}{\textit{The Proposed Methods with Perceptual Entropy $\mathcal{H}_{\mathrm{perc}}$ from DINO Space}} \\
\midrule
\rowcolor{gray!20} \textit{Perc-VC} (Ours, $\lambda=0.03$) & DINO & \textcolor{blue}{\underline{0.244}} & \textcolor{blue}{6.362} & \textcolor{blue}{\textbf{1.497}} & \textcolor{blue}{0.405} & \textcolor{blue}{0.344} & \textcolor{blue}{\underline{3.774}} & \cellcolor{cvprblue!54}\textcolor{blue}{\underline{0.876}} & \textcolor{blue}{0.64} & \textcolor{blue}{0.43} & \textcolor{blue}{4.06} & \textcolor{blue}{2.15} & \textcolor{blue}{1.21} & \textcolor{blue}{1.14} & \cellcolor{darkgreen!24}\textcolor{blue}{0.241} & \cellcolor{orange!38}\textcolor{blue}{0.559} \\
\rowcolor{gray!20} \textit{Perc-VC} (Ours, $\lambda=0.10$) & DINO & \textcolor{blue}{0.237} & \textcolor{blue}{6.176} & \textcolor{blue}{1.328} & \textcolor{blue}{0.401} & \textcolor{blue}{0.332} & \textcolor{blue}{3.665} & \cellcolor{cvprblue!38}\textcolor{blue}{0.546} & \textcolor{blue}{\textbf{0.78}} & \textcolor{blue}{0.61} & \textcolor{blue}{\textbf{6.88}} & \textcolor{blue}{3.94} & \textcolor{blue}{1.35} & \textcolor{blue}{1.21} & \cellcolor{darkgreen!44}\textcolor{blue}{0.680} & \cellcolor{orange!42}\textcolor{blue}{0.613} \\
\rowcolor{gray!20} \textit{Perc-EC} (Ours, $\lambda=0.05$) & DINO & \textcolor{blue}{0.238} & \textcolor{blue}{\underline{6.405}} & \textcolor{blue}{1.428} & \textcolor{blue}{0.406} & \textcolor{blue}{0.336} & \textcolor{blue}{3.717} & \cellcolor{cvprblue!48}\textcolor{blue}{0.746} & \textcolor{blue}{0.70} & \textcolor{blue}{0.52} & \textcolor{blue}{4.83} & \textcolor{blue}{2.86} & \textcolor{blue}{1.30} & \textcolor{blue}{1.16} & \cellcolor{darkgreen!32}\textcolor{blue}{0.420} & \cellcolor{orange!40}\textcolor{blue}{0.583} \\
\rowcolor{gray!20} \textit{Perc-EC} (Ours, $\lambda=0.10$) & DINO & \textcolor{blue}{0.233} & \textcolor{blue}{6.195} & \textcolor{blue}{1.327} & \textcolor{blue}{0.396} & \textcolor{blue}{0.332} & \textcolor{blue}{3.640} & \cellcolor{cvprblue!35}\textcolor{blue}{0.468} & \textcolor{blue}{\textbf{0.78}} & \textcolor{blue}{0.67} & \textcolor{blue}{\underline{6.72}} & \textcolor{blue}{5.18} & \textcolor{blue}{1.41} & \textcolor{blue}{1.27} & \cellcolor{darkgreen!50}\textcolor{blue}{0.802} & \cellcolor{orange!42}\textcolor{blue}{0.635} \\
\midrule
\multicolumn{17}{c}{\textit{The Proposed Methods with Perceptual Entropy $\mathcal{H}_{\mathrm{perc}}$ from PickScore Space}} \\
\midrule
\rowcolor{gray!20} \textit{Perc-VC} (Ours, $\lambda=0.03$) & PickScore & \textcolor{blue}{\textbf{0.246}} & \textcolor{blue}{6.388} & \textcolor{blue}{\underline{1.489}} & \textcolor{blue}{\underline{0.408}} & \textcolor{blue}{\textbf{0.348}} & \textcolor{blue}{\textbf{3.816}} & \cellcolor{cvprblue!58}\textcolor{blue}{\textbf{0.962}} & \textcolor{blue}{0.61} & \textcolor{blue}{0.44} & \textcolor{blue}{3.61} & \textcolor{blue}{2.17} & \textcolor{blue}{1.22} & \textcolor{blue}{1.13} & \cellcolor{darkgreen!22}\textcolor{blue}{0.204} & \cellcolor{orange!40}\textcolor{blue}{0.583} \\
\rowcolor{gray!20} \textit{Perc-VC} (Ours, $\lambda=0.10$) & PickScore & \textcolor{blue}{0.240} & \textcolor{blue}{6.211} & \textcolor{blue}{1.335} & \textcolor{blue}{0.397} & \textcolor{blue}{0.335} & \textcolor{blue}{3.683} & \cellcolor{cvprblue!40}\textcolor{blue}{0.581} & \textcolor{blue}{\underline{0.73}} & \textcolor{blue}{\textbf{0.80}} & \textcolor{blue}{5.44} & \textcolor{blue}{\textbf{7.33}} & \textcolor{blue}{\underline{1.42}} & \textcolor{blue}{\underline{1.31}} & \cellcolor{darkgreen!53}\textcolor{blue}{\underline{0.862}} & \cellcolor{orange!47}\textcolor{blue}{\underline{0.721}} \\
\rowcolor{gray!20} \textit{Perc-EC} (Ours, $\lambda=0.05$) & PickScore & \textcolor{blue}{0.240} & \textcolor{blue}{\textbf{6.412}} & \textcolor{blue}{1.446} & \textcolor{blue}{\textbf{0.412}} & \textcolor{blue}{0.338} & \textcolor{blue}{3.748} & \cellcolor{cvprblue!53}\textcolor{blue}{0.842} & \textcolor{blue}{0.66} & \textcolor{blue}{0.52} & \textcolor{blue}{4.10} & \textcolor{blue}{2.71} & \textcolor{blue}{1.26} & \textcolor{blue}{1.15} & \cellcolor{darkgreen!28}\textcolor{blue}{0.335} & \cellcolor{orange!41}\textcolor{blue}{0.589} \\
\rowcolor{gray!20} \textit{Perc-EC} (Ours, $\lambda=0.10$) & PickScore & \textcolor{blue}{0.235} & \textcolor{blue}{6.242} & \textcolor{blue}{1.340} & \textcolor{blue}{0.401} & \textcolor{blue}{0.334} & \textcolor{blue}{3.674} & \cellcolor{cvprblue!40}\textcolor{blue}{0.568} & \textcolor{blue}{\textbf{0.78}} & \textcolor{blue}{\underline{0.74}} & \textcolor{blue}{6.41} & \textcolor{blue}{\underline{6.73}} & \textcolor{blue}{\textbf{1.51}} & \textcolor{blue}{\textbf{1.32}} & \cellcolor{darkgreen!58}\textcolor{blue}{\textbf{0.941}} & \cellcolor{orange!48}\textcolor{blue}{\textbf{0.755}} \\
\bottomrule
\end{tabular}
}
\endgroup
\end{table}

%% file: tables/sota.tex
\begin{table}[t!]
\begingroup
\colorlet{blue}{black}
\centering
\caption{\textbf{SoTA Comparison.}  Type denotes post-training: N=None, U=Unknown, D=DPO, and R=reinforcement fine-tuning.}
\label{tab:sota_results}
\resizebox{\textwidth}{!}{
\small
\setlength{\tabcolsep}{2pt}
\begin{tabular}{llccccccccccccccc}
\toprule
\multirow{2}{*}{Method} & \multirow{2}{*}{Type} & \multicolumn{7}{c}{Quality$\uparrow$} & \multicolumn{7}{c}{Diversity$\uparrow$} & \multirow{2}{*}{Overall} \\
\cmidrule(lr){3-9} \cmidrule(lr){10-16}
  &   &  PickScore  &  Aes.v2.5  &  ImageReward  &  CLIP  &  HPSv2.1  &  Unified Reward  &  Norm.Avg  & \textcolor{blue}{DINO} & \textcolor{blue}{CLIP} & $\mathrm{V.S.}_{\mathrm{DINO}}$ & $\mathrm{V.S.}_{\mathrm{CLIP}}$ & $\mathrm{V.S.}_{\mathrm{IR}}$ & $\mathrm{V.S.}_{\mathrm{PS}}$ & Norm.Avg &  \\
\midrule
SDXL~\cite{podell2023sdxl} & N & 0.226 & 5.836 & 0.980 & 0.411 & 0.294 & 3.410 & \cellcolor{cvprblue!22}0.205 & \textcolor{blue}{\textbf{0.83}} & \textcolor{blue}{0.54} & \textcolor{blue}{\textbf{8.50}} & \textcolor{blue}{3.06} & \textcolor{blue}{\underline{1.50}} & \textcolor{blue}{1.14} & \cellcolor{darkgreen!43}\textcolor{blue}{0.651} & \cellcolor{orange!32}\textcolor{blue}{0.428} \\
Hunyuan-DiT~\cite{kong2024hunyuandit} & U & 0.224 & 5.659 & 1.080 & 0.406 & 0.303 & 3.428 & \cellcolor{cvprblue!22}0.190 & \textcolor{blue}{0.74} & \textcolor{blue}{0.51} & \textcolor{blue}{6.22} & \textcolor{blue}{2.92} & \textcolor{blue}{1.32} & \textcolor{blue}{1.13} & \cellcolor{darkgreen!34}\textcolor{blue}{0.438} & \cellcolor{orange!28}\textcolor{blue}{0.314} \\
SD-3.5-M~\cite{sd3} & D & 0.227 & 5.711 & 1.113 & 0.404 & 0.301 & 3.553 & \cellcolor{cvprblue!24}0.256 & \textcolor{blue}{0.76} & \textcolor{blue}{0.51} & \textcolor{blue}{6.78} & \textcolor{blue}{2.85} & \textcolor{blue}{1.33} & \textcolor{blue}{1.13} & \cellcolor{darkgreen!35}\textcolor{blue}{0.467} & \cellcolor{orange!29}\textcolor{blue}{0.361} \\
Kolors~\cite{kolors} & N & 0.225 & 5.994 & 0.977 & 0.385 & 0.314 & 3.362 & \cellcolor{cvprblue!19}0.143 & \textcolor{blue}{\underline{0.78}} & \textcolor{blue}{0.51} & \textcolor{blue}{\underline{7.59}} & \textcolor{blue}{2.83} & \textcolor{blue}{1.40} & \textcolor{blue}{1.16} & \cellcolor{darkgreen!38}\textcolor{blue}{0.550} & \cellcolor{orange!29}\textcolor{blue}{0.347} \\
SD-3.5-L~\cite{sd3} & D & 0.229 & 5.897 & 1.122 & 0.408 & 0.302 & 3.619 & \cellcolor{cvprblue!29}0.359 & \textcolor{blue}{0.71} & \textcolor{blue}{0.44} & \textcolor{blue}{5.53} & \textcolor{blue}{2.25} & \textcolor{blue}{1.22} & \textcolor{blue}{1.11} & \cellcolor{darkgreen!26}\textcolor{blue}{0.299} & \cellcolor{orange!28}\textcolor{blue}{0.329} \\
HiDream-Dev~\cite{hidreami1technicalreport} & U & 0.229 & 6.049 & 1.329 & 0.398 & 0.325 & 3.781 & \cellcolor{cvprblue!37}0.527 & \textcolor{blue}{0.57} & \textcolor{blue}{0.39} & \textcolor{blue}{3.58} & \textcolor{blue}{2.08} & \textcolor{blue}{1.19} & \textcolor{blue}{1.10} & \cellcolor{darkgreen!18}\textcolor{blue}{0.124} & \cellcolor{orange!28}\textcolor{blue}{0.326} \\
Qwen-Image~\cite{wu2025qwenimagetechnicalreport} & D\&R & 0.232 & 6.039 & 1.409 & \textbf{0.417} & 0.324 & \textbf{3.921} & \cellcolor{cvprblue!46}0.711 & \textcolor{blue}{0.52} & \textcolor{blue}{0.35} & \textcolor{blue}{2.77} & \textcolor{blue}{1.78} & \textcolor{blue}{1.12} & \textcolor{blue}{1.07} & \cellcolor{darkgreen!12}\textcolor{blue}{0.000} & \cellcolor{orange!29}\textcolor{blue}{0.356} \\
\midrule
Flux Dev \textit{(Base Model)} & U & 0.226 & 5.837 & 1.089 & 0.388 & 0.312 & 3.540 & \cellcolor{cvprblue!23}0.215 & \textcolor{blue}{0.70} & \textcolor{blue}{0.49} & \textcolor{blue}{5.14} & \textcolor{blue}{2.61} & \textcolor{blue}{1.26} & \textcolor{blue}{1.15} & \cellcolor{darkgreen!29}\textcolor{blue}{0.356} & \cellcolor{orange!25}\textcolor{blue}{0.285} \\
\rowcolor{gray!20} \textit{Perc-VC} (Ours, $\lambda=0.03$) & R & \textcolor{blue}{\textbf{0.246}} & \textcolor{blue}{\underline{6.388}} & \textcolor{blue}{\textbf{1.489}} & \textcolor{blue}{0.408} & \textcolor{blue}{\textbf{0.348}} & \textcolor{blue}{\underline{3.816}} & \cellcolor{cvprblue!56}\textcolor{blue}{\textbf{0.917}} & \textcolor{blue}{0.61} & \textcolor{blue}{0.44} & \textcolor{blue}{3.61} & \textcolor{blue}{2.17} & \textcolor{blue}{1.22} & \textcolor{blue}{1.13} & \cellcolor{darkgreen!22}\textcolor{blue}{0.201} & \cellcolor{orange!38}\textcolor{blue}{0.559} \\
\rowcolor{gray!20} \textit{Perc-VC} (Ours, $\lambda=0.10$) & R & \textcolor{blue}{\underline{0.240}} & \textcolor{blue}{6.211} & \textcolor{blue}{1.335} & \textcolor{blue}{0.397} & \textcolor{blue}{0.335} & \textcolor{blue}{3.683} & \cellcolor{cvprblue!43}\textcolor{blue}{0.645} & \textcolor{blue}{0.73} & \textcolor{blue}{\textbf{0.80}} & \textcolor{blue}{5.44} & \textcolor{blue}{\textbf{7.33}} & \textcolor{blue}{1.42} & \textcolor{blue}{\underline{1.31}} & \cellcolor{darkgreen!50}\textcolor{blue}{\underline{0.812}} & \cellcolor{orange!47}\textcolor{blue}{\underline{0.729}} \\
\rowcolor{gray!20} \textit{Perc-EC} (Ours, $\lambda=0.05$) & R & \textcolor{blue}{\underline{0.240}} & \textcolor{blue}{\textbf{6.412}} & \textcolor{blue}{\underline{1.446}} & \textcolor{blue}{\underline{0.412}} & \textcolor{blue}{\underline{0.338}} & \textcolor{blue}{3.748} & \cellcolor{cvprblue!52}\textcolor{blue}{\underline{0.832}} & \textcolor{blue}{0.66} & \textcolor{blue}{0.52} & \textcolor{blue}{4.10} & \textcolor{blue}{2.71} & \textcolor{blue}{1.26} & \textcolor{blue}{1.15} & \cellcolor{darkgreen!28}\textcolor{blue}{0.318} & \cellcolor{orange!40}\textcolor{blue}{0.575} \\
\rowcolor{gray!20} \textit{Perc-EC} (Ours, $\lambda=0.10$) & R & \textcolor{blue}{0.235} & \textcolor{blue}{6.242} & \textcolor{blue}{1.340} & \textcolor{blue}{0.401} & \textcolor{blue}{0.334} & \textcolor{blue}{3.674} & \cellcolor{cvprblue!42}\textcolor{blue}{0.630} & \textcolor{blue}{\underline{0.78}} & \textcolor{blue}{\underline{0.74}} & \textcolor{blue}{6.41} & \textcolor{blue}{\underline{6.73}} & \textcolor{blue}{\textbf{1.51}} & \textcolor{blue}{\textbf{1.32}} & \cellcolor{darkgreen!54}\textcolor{blue}{\textbf{0.872}} & \cellcolor{orange!48}\textcolor{blue}{\textbf{0.751}} \\
\bottomrule
\end{tabular}
}
\endgroup
\end{table}

%% file: tables/ablation_perc_entropy.tex
\centering
\begingroup
\captionof{table}{\textbf{Perceptual Entropy Ablation.} Our perceptual entropy $\mathcal{H}_{\mathrm{perc}}$ for \textit{Perc-EC} is replaced with native entropy $\mathcal{H}$.}
\label{tab:ablation_perc_entropy}
\resizebox{\linewidth}{!}{
\small
\begin{tabular}{llcccc}
\toprule
Method & Entropy & Aes.v2.5$\uparrow$ & ImageReward$\uparrow$ & CLIP$\uparrow$ & DINO$\uparrow$ \\
\midrule
\textit{Perc-EC} & $\mathcal{H}$ & 6.372 & 1.439 & 0.405 & 0.61 \\
\rowcolor{gray!20}{\textit{Perc-EC}} & \textbf{$\mathcal{H}_\mathrm{perc}$} & \textbf{6.412} & \textbf{1.446} & \textbf{0.412} & \textbf{0.66} \\
\bottomrule
\end{tabular}
}
\endgroup

%% file: tables/runtime_overhead.tex
\centering
\captionof{table}{{\textbf{Computational Cost.} Iteration time is normalized by Flow-GRPO. Cost is limited.}}
\label{tab:runtime_overhead}
\resizebox{\linewidth}{!}{
\small
\setlength{\tabcolsep}{4pt}
{
\begin{tabular}{lccc}
\toprule
Method & Perceptual Encoder & Iteration Time & Peak Memory \\
\midrule
Flow-GRPO    & --          & $1.00\times$ & 78.7G \\
\rowcolor{gray!20}{\textit{Perc-EC} (Ours)}   & PickScore Encoder   & $1.01\times$ & 78.7G \\
\rowcolor{gray!20}{\textit{Perc-EC} (Ours)}  & DINO Encoder        & $1.03\times$ & 79.2G \\
\bottomrule
\end{tabular}
}}

%% file: sections/7_related_works.tex
\section{Related Work}

\noindent \textbf{RLHF for Flow Matching.}
Controllable visual generation such as text-to-image~\citep{liu2025flow,unirl_github}, text-to-motion~\citep{tan2026motionrft,tan2026easytune,wang2026realign, weng2026bimogen}, and drag-based image editing~\citep{pu2026dragevidencemotiongroundedlatent,pu2026dragginggeometrypixelsgeometryguided} require human preference. Thus, RLHF optimizes generative models via reward signals~\citep{liu2026alignment}. Building on early progress for diffusion models~\citep{black2023training,fan2023dpok,wallace2024diffusion,yang2024using}, recent work extends RLHF to flow matching by enabling stochastic rollouts and GRPO-style optimization~\citep{xue2025dancegrpo,liu2025flow,tan2026consistentrft}. Follow-up studies improve efficiency~\citep{he2025tempflow,fu2025dynamic,li2025branchgrpo}, preference modeling~\citep{zheng2025diffusionnft,tan2025sopo}, theory~\citep{wang2025coefficients}, and reward hacking~\citep{wang2025grpo,tan2026consistentrft}. Despite these advances, RLHF for flow models often suffers from diversity collapse, and its underlying mechanism remains unclear.

\noindent \textbf{Entropy in LLM Reinforcement Finetuning.}
Entropy is widely used to characterize exploration and predict downstream gains in RLVR~\citep{haarnoja2018soft,schulman2017proximal,cui2025entropy}. Token-level analyses further show that high-entropy tokens often correspond to key forking points and drive a disproportionate share of learning~\citep{bigelow2024forking,wang2025beyond}. Prior work thus studies entropy-aware objectives and regularization, together with stabilizers such as loss reweighting and clip-higher, to mitigate entropy collapse~\citep{vanlioglu2025entropy,cheng2025reasoning,liu2025prorl,yu2025dapo}. In contrast, entropy in flow models can behave substantially differently, yet remains underexplored.

\noindent \textbf{Entropy in Diffusion and Flow Models.} Prior work studies entropy-regularized diffusion fine-tuning~\citep{uehara2024understanding}. DIME~\citep{celik2025diffmaxent} incorporates an entropy bound into critic targets, and FMER~\citep{gao2026fmer} combines a critic with a differentiated velocity-divergence estimator. S-MEME and FDC~\citep{desanti2025provable,desanti2025fdc} rely on optimal-control solvers requiring gradients of functional first variations, while ECFM~\citep{maduabuchi2026ecfm} constrains entropy rates along the flow. S$^2$AC~\citep{messaoud2024s2ac} estimates energy-based policy entropy via SVGD change of variables, and MET-SVGD~\citep{messaoud2026particles} addresses its high-dimensional estimation limitations. In contrast, \textit{Perc-EC} and \textit{Perc-VC} compute perceptual entropy $\mathcal{H}_{\mathrm{perc}}$ in a frozen encoder's feature space, rather than low-level VAE space $\mathcal{H}$, and apply it as a shaped reward on fixed rollouts, requiring no learned critic, score differentiation, or reward gradients.

%% file: sections/6_conclusion.tex
\section{Conclusion}
\label{sec:conclusion}

In this work, we investigate diversity collapse in flow-based RLHF from the perspective of entropy mechanism, and propose a novel perceptual entropy framework with two entropy-regularization strategies that improve alignment without sacrificing diversity. First, we identify and analyze a flow-specific paradox in which perceptual diversity collapses despite the practical policy entropy surrogate remaining nearly constant. We show that this phenomenon arises from the fixed reverse-process noise schedule, whereas diversity collapse is fundamentally driven by the mode-seeking nature of policy-gradient optimization. Second, we introduce perceptual entropy, which captures diversity in a perceptual feature space while preserving key properties of standard entropy, and further propose \textit{Perc-EC} and \textit{Perc-VC} to maintain perceptual diversity. Extensive experiments on FLUX.dev demonstrate that our method substantially improves the quality-diversity trade-off.

%% file: sections/8_appendix.tex
\begin{figure}[t]
    \centering
    \includegraphics[width=1.0\textwidth]{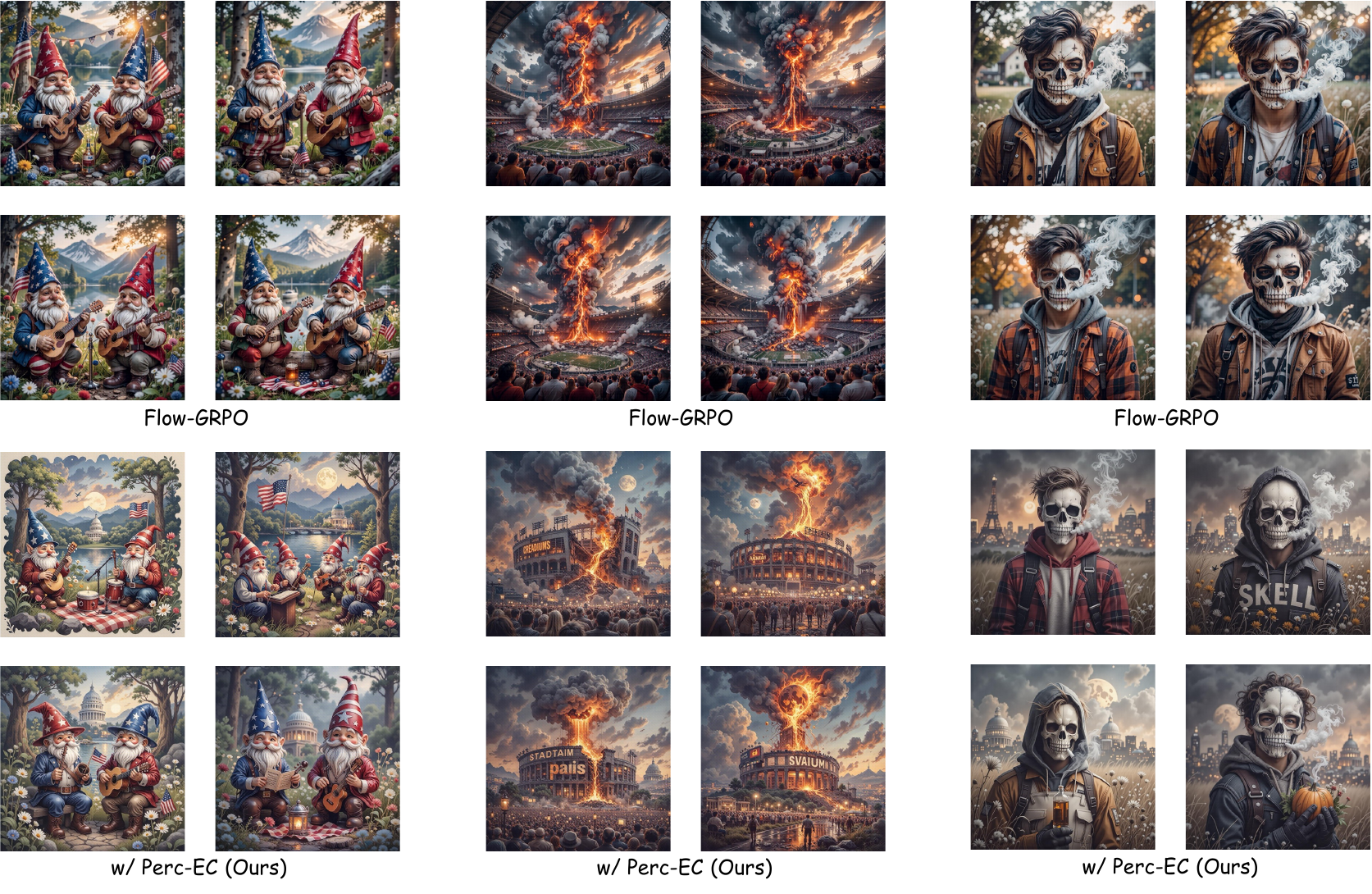}
    \caption{\textbf{Extended Diversity Visualization.} \textit{Perc-EC} preserves visible variation across prompts and visual domains.}
    \label{fig:app_vis}
\end{figure}

\section{Experiments}
\label{app:experiments}

\subsection{Extended Visualization and User Study}
\label{app:visualization}
\label{app:user_study}

Our method demonstrates strong diversity preservation across diverse prompts and domains. Fig.~\ref{fig:app_vis} presents an extensive collection of generated images showcasing the diversity of our approach. \textit{Perc-EC} effectively balances quality and diversity: images maintain high aesthetic and semantic quality while exhibiting substantial visual variation within each prompt category, demonstrating that our perceptual entropy constraint successfully mitigates diversity collapse.

To further validate the effectiveness of our approach, we conduct a user study comparing our \textit{Perc-EC} method with Flow-GRPO. The study follows a blind, randomized, side-by-side protocol: participants are shown two anonymized image sets generated by the two methods in random order and are asked to choose the better set, or indicate no clear preference, under three criteria: overall preference, diversity, and quality.
We collect 100 annotations over 20 unique image-set pairs. Specifically, we recruited 20 volunteer participants from graduate students and researchers familiar with text-to-image generation. Each participant annotated five randomized image-set pairs, yielding five independent annotations per unique pair. As shown in Fig.~\ref{fig:user_study}, \textit{Perc-EC} is preferred over Flow-GRPO in 64\% of overall judgments and 82\% of diversity judgments, with 20\% and 10\% no-preference responses, respectively. For image quality, participants rate the two methods comparably, with 38\% preferring \textit{Perc-EC}, 43\% preferring Flow-GRPO, and 19\% indicating no clear preference. The results demonstrate that our perceptual entropy constraint improves subjective diversity and overall preference while maintaining comparable image quality.

\begin{figure}[t]
    \centering
    \begin{minipage}[b]{0.3\textwidth}
        \centering
        \includegraphics[width=\textwidth]{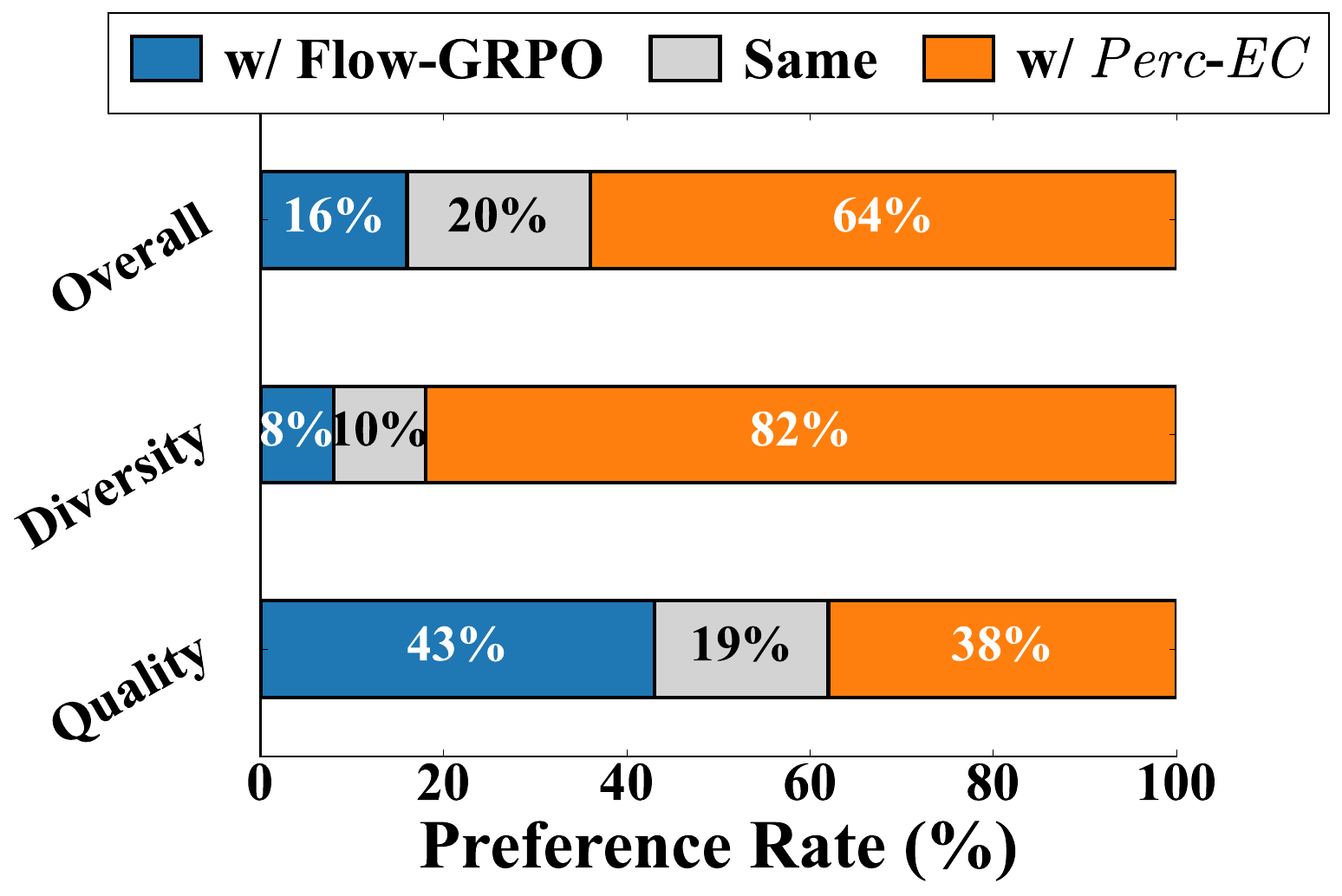}
        \caption{\textbf{User Study Preferences.} }
        \label{fig:user_study}
    \end{minipage}
    \hfill
    \begin{minipage}[b]{0.33\textwidth}
        \centering
        \includegraphics[width=\textwidth]{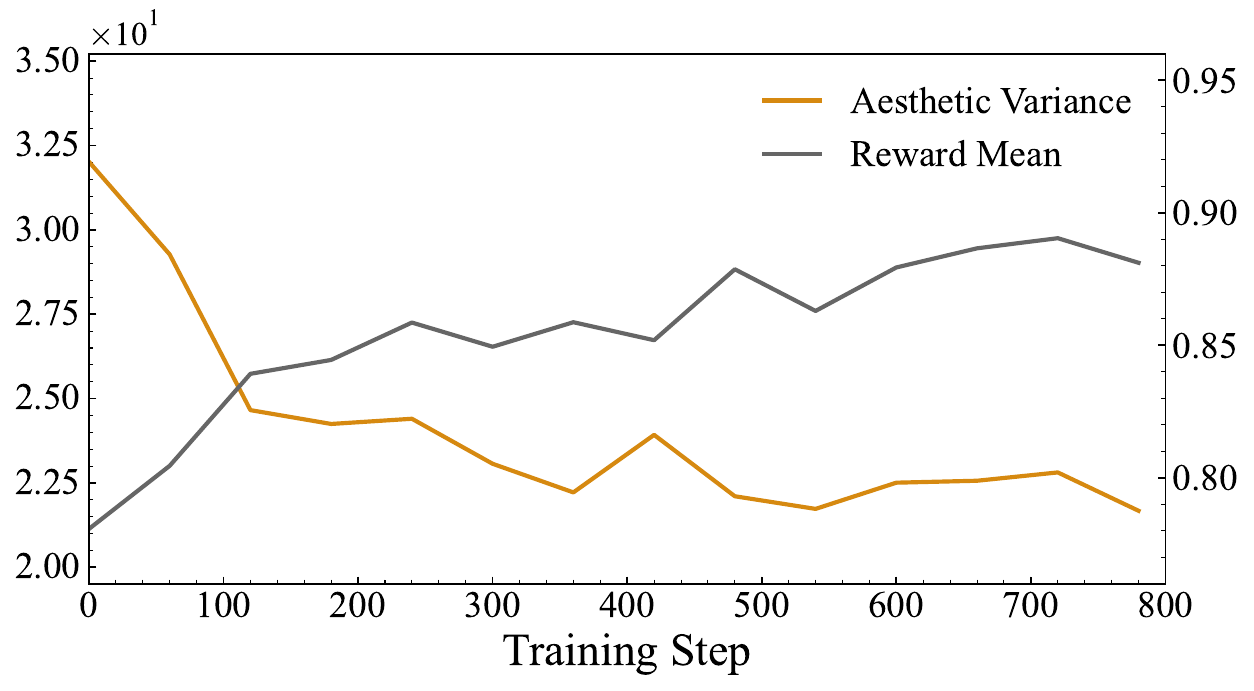}
        \caption{\textbf{Aesthetic Space Variance.} }
        \label{fig:variance_aesthetic}
    \end{minipage}
    \hfill
    \begin{minipage}[b]{0.33\textwidth}
        \centering
        \includegraphics[width=\textwidth]{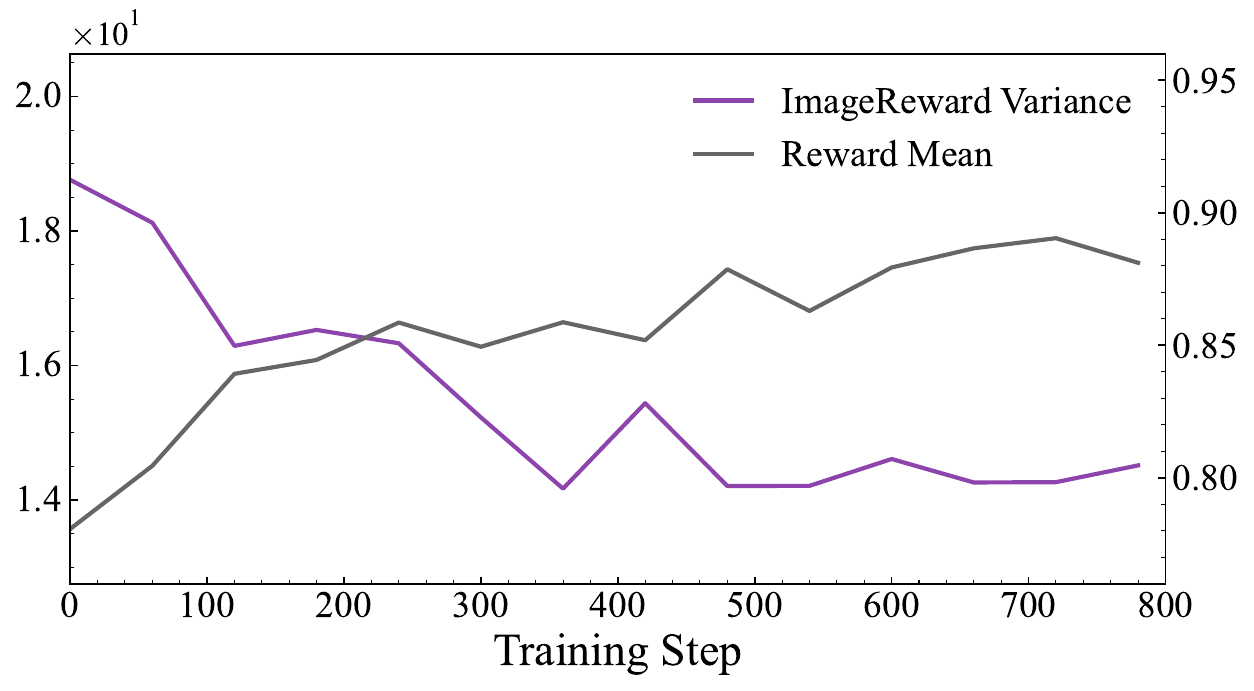}
        \caption{\textbf{ImageReward Space Variance.} }
        \label{fig:variance_imagereward}
    \end{minipage}
\end{figure}

\subsection{Reward Variance Analysis on Other Perceptual Spaces}
\label{app:reward_variance}

To verify that our findings generalize beyond the default perceptual space, we analyze reward variance dynamics in alternative perceptual spaces defined by the Aesthetic and ImageReward encoders. As shown in Figs.~\ref{fig:variance_aesthetic} and~\ref{fig:variance_imagereward}, the same diversity-quality trade-off pattern emerges across different perceptual spaces: vanilla Flow-GRPO progressively reduces reward variance during training, reflecting the implicit drive toward mode collapse identified in our theoretical analysis. These results confirm that the entropy-variance relationship established in our framework is not an artifact of a particular feature space but a general property of policy gradient-based flow RLHF.

\subsection{Rule-Based Reward Results on SD3.5-M}
\label{app:sd35_rule_based}

To complement the FLUX.dev experiments with neural rewards, we further consider the rule-based reward setting used in Flow-GRPO~\citep{liu2025flow}. For completeness, Tab.~\ref{tab:app_flowgrpo_table1} reproduces the GenEval comparison table from Flow-GRPO and includes our \textit{Perc-EC} variant under the same SD3.5-M family setting.

\begin{table*}[t]
\centering
\caption{\textbf{GenEval Result.} Results for models other than SD3.5-M are from~\citep{Gpt-imgeval} or the corresponding original papers. Obj.: Object; Attr.: Attribute binding.}
\label{tab:app_flowgrpo_table1}
\resizebox{0.85\textwidth}{!}{
\small
\begin{tabular}{l|c|cccccc}
\toprule
\textbf{Model} & \textbf{Overall} & \textbf{Single Obj.} & \textbf{Two Obj.} & \textbf{Counting} & \textbf{Colors} & \textbf{Position} & \textbf{Attr. Binding} \\
\midrule
\multicolumn{8}{c}{\textit{Diffusion Models}} \\
\midrule
LDM~\citep{rombach2022high} & 0.37 & 0.92 & 0.29 & 0.23 & 0.70 & 0.02 & 0.05 \\
SD1.5~\citep{rombach2022high} & 0.43 & 0.97 & 0.38 & 0.35 & 0.76 & 0.04 & 0.06 \\
SD2.1~\citep{rombach2022high} & 0.50 & 0.98 & 0.51 & 0.44 & 0.85 & 0.07 & 0.17 \\
SD-XL~\citep{podell2023sdxl} & 0.55 & 0.98 & 0.74 & 0.39 & 0.85 & 0.15 & 0.23 \\
DALLE-2~\citep{ramesh2022hierarchical} & 0.52 & 0.94 & 0.66 & 0.49 & 0.77 & 0.10 & 0.19 \\
DALLE-3~\citep{betker2023improving} & 0.67 & 0.96 & 0.87 & 0.47 & 0.83 & 0.43 & 0.45 \\
\midrule
\multicolumn{8}{c}{\textit{Autoregressive Models}} \\
\midrule
Show-o~\citep{xie2024show} & 0.53 & 0.95 & 0.52 & 0.49 & 0.82 & 0.11 & 0.28 \\
Emu3-Gen~\citep{wang2024emu3} & 0.54 & 0.98 & 0.71 & 0.34 & 0.81 & 0.17 & 0.21 \\
JanusFlow~\citep{ma2024janusflow} & 0.63 & 0.97 & 0.59 & 0.45 & 0.83 & 0.53 & 0.42 \\
Janus-Pro-7B~\citep{chen2025janus} & 0.80 & 0.99 & 0.89 & 0.59 & 0.90 & 0.79 & 0.66 \\
GPT-4o~\citep{gpt4o} & 0.84 & 0.99 & 0.92 & 0.85 & 0.92 & 0.75 & 0.61 \\
\midrule
\multicolumn{8}{c}{\textit{Flow Matching Models}} \\
\midrule
FLUX.1 Dev~\citep{flux2024} & 0.66 & 0.98 & 0.81 & 0.74 & 0.79 & 0.22 & 0.45 \\
SD3.5-L~\citep{sd3} & 0.71 & 0.98 & 0.89 & 0.73 & 0.83 & 0.34 & 0.47 \\
SANA-1.5 4.8B~\citep{xie2025sana} & 0.81 & 0.99 & 0.93 & 0.86 & 0.84 & 0.59 & 0.65 \\
SD3.5-M~\citep{sd3} & 0.63 & 0.98 & 0.78 & 0.50 & 0.81 & 0.24 & 0.52 \\
\midrule
SD3.5-M+Flow-GRPO & 0.95 & \textbf{1.00} & \textbf{0.99} & 0.95 & 0.92 & 0.99 & 0.86 \\
\rowcolor{gray!20} \textbf{+\textit{Perc-EC} (Ours)} & \textbf{0.97} & \textbf{1.00} & \textbf{0.99} & \textbf{0.97} & \textbf{0.94} & \textbf{1.00} & \textbf{0.91} \\
\bottomrule
\end{tabular}
}
\end{table*}

\section{Theoretical Proofs and Derivations}
\label{app:theoretical}

\subsection{Proof of Property~\ref{prop:constant-entropy}}
\label{app:proof1}

In this section, we establish Property~\ref{prop:constant-entropy}: under the reverse transition in Sec.~\ref{sec:preliminary}, per-step conditional entropy is independent of $\theta$. We restate the claim for convenience.

\begin{property*}[Constant Per-step Entropy in Flow-based RL]
For the reverse-time transition defined in Eq.~\eqref{eq:x_tdt_rewrite}, the per-step conditional entropy at timestep $t$ satisfies:
\begin{equation*}
\mathcal{H}_t(p_\theta) = \frac{d}{2} - C_t.
\end{equation*}
where $d$ is the dimension of $\mathbf{x}_t$, and $C_t$ is a constant depending on timestep $t$.
\end{property*}

\begin{proof}
By substituting Eqs.~\eqref{eq:x_tdt_rewrite} and the log-probability definition into Eq.~\eqref{eq:entropy-in-flow-matching-rl}, the following key property can be derived.

We analyze the entropy term $\mathcal{H}(p_\theta) = -\mathbb{E}_{p_\theta} [\log p_\theta]$ in the gradient computation. To simplify the analysis, we directly consider sampling from $p_\theta$.

For a continuous Gaussian transition $p_\theta(\mathbf{x}_{t-1}|\mathbf{x}_t, c)$ at fixed timestep $t$, the per-step conditional entropy is
\begin{equation}
\mathcal{H}_t(p_\theta) = -\mathbb{E}_{\mathbf{x}_{t-1}, \mathbf{x}_t} \left[ \log p_\theta(\mathbf{x}_{t-1}|\mathbf{x}_t, c) \right].
\end{equation}

Since $p_\theta(\mathbf{x}_{t-1}|\mathbf{x}_t, c) = \mathcal{N}(\mu_\theta(\mathbf{x}_t, t), \sigma_t^2 \mathrm{d}t \, \mathbf{I})$, we have $\mathbf{x}_{t-1} = \mu_\theta(\mathbf{x}_t, t) + \sigma_t \sqrt{\mathrm{d}t} \cdot \boldsymbol{\epsilon}$ where $\boldsymbol{\epsilon} \sim \mathcal{N}(\mathbf{0}, \mathbf{I})$.

The log-probability is:
\begin{equation}
\log p_\theta(\mathbf{x}_{t-1}|\mathbf{x}_t, c) = -\frac{\|\mathbf{x}_{t-1} - \mu_\theta(\mathbf{x}_t, t)\|^2}{2\sigma_t^2 \mathrm{d}t} + C_t
\end{equation}

Substituting the sampling formula $\mathbf{x}_{t-1} = \mu_\theta(\mathbf{x}_t, t) + \sigma_t \sqrt{\mathrm{d}t} \cdot \boldsymbol{\epsilon}$ gives
\begin{equation}
\log p_\theta(\mathbf{x}_{t-1}|\mathbf{x}_t, c) = -\frac{\|\sigma_t \sqrt{\mathrm{d}t} \cdot \boldsymbol{\epsilon}\|^2}{2\sigma_t^2 \mathrm{d}t} + C_t = -\frac{\|\boldsymbol{\epsilon}\|^2}{2} + C_t.
\end{equation}

Taking the expectation over $\boldsymbol{\epsilon} \sim \mathcal{N}(\mathbf{0}, \mathbf{I})$ yields
\begin{equation}
\mathcal{H}_t(p_\theta) = -\mathbb{E}_{\boldsymbol{\epsilon}} \left[ -\frac{\|\boldsymbol{\epsilon}\|^2}{2} + C_t \right] = \mathbb{E}_{\boldsymbol{\epsilon}} \left[ \frac{\|\boldsymbol{\epsilon}\|^2}{2} \right] - C_t = \frac{d}{2} - C_t.
\end{equation}

where we used $\mathbb{E}[\|\boldsymbol{\epsilon}\|^2] = d$. Here $\|\boldsymbol{\epsilon}\|^2$ follows a chi-squared distribution $\chi^2(d)$ with $d$ degrees of freedom, so its mean is $d$.

Therefore, we conclude:
\begin{equation}
\boxed{\mathcal{H}_t(p_\theta) = \frac{d}{2} - C_t}
\end{equation}

This completes the proof.
\end{proof}

\subsection{Proof of Latent-Space Conditional Variance Interpretation}
\label{app:proof2}

In this section, we connect the per-step entropy $\mathcal{H}(p_\theta)$ in flow matching to the expected conditional variance of sampled latent variables $\mathbf{x}_{t-1}$ given $\mathbf{x}_t$. Consistent with prior work~\citep{liu2025flow,xue2025dancegrpo}, $\mathbf{x}$ refers to samples in the flow model's latent space rather than images in pixel space. We use $\mathrm{Var}(\mathbf{x}_{t-1}\mid \mathbf{x}_t)$ as a shorthand for the scalar $\mathrm{tr}(\mathrm{Cov}(\mathbf{x}_{t-1}\mid \mathbf{x}_t))$, i.e., the (sum of marginal) conditional variance of $\mathbf{x}_{t-1}$ given $\mathbf{x}_t$. We restate the claim formally as a Lemma.

\begin{lemma*}[Conditional Variance Interpretation of Per-Step Entropy]
For the reverse-time Gaussian transition in Eq.~\eqref{eq:x_tdt_rewrite}, the per-step entropy in flow matching satisfies
\begin{equation*}
\mathcal{H}(p_\theta) =
\mathbb{E}_{t,\mathbf{x}_t}\!\left[
\frac{\mathrm{Var}(\mathbf{x}_{t-1}\mid \mathbf{x}_t)}
{2\sigma_t^2 \mathrm{d}t}- C_t\right],
\end{equation*}
where $C_t=-\frac{d}{2}\log(2\pi\sigma_t^2\mathrm{d}t)$ is a timestep-dependent constant independent of $\theta$.
\end{lemma*}

\begin{proof}
We begin by analyzing the entropy of the conditional distribution $p_\theta(\mathbf{x}_{t-1}\mid\mathbf{x}_t, c)$. The derivation follows the same approach as Property~\ref{prop:constant-entropy}. Recall that
\begin{equation}
p_\theta(\mathbf{x}_{t-1}\mid\mathbf{x}_t, c) = \mathcal{N}(\mathbf{x}_{t-1};\, \mu_\theta(\mathbf{x}_t, t),\; \sigma_t^2 \mathrm{d}t \cdot \mathbf{I}),
\end{equation}
with the corresponding log-probability
\begin{equation}
\log p_\theta(\mathbf{x}_{t-1}\mid\mathbf{x}_t, c) = -\frac{\|\mathbf{x}_{t-1} - \mu_\theta(\mathbf{x}_t, t)\|^2}{2\sigma_t^2 \mathrm{d}t} + C_t.
\end{equation}

The differential entropy is then
\begin{equation}
\mathcal{H}(p_\theta)
= -\mathbb{E}_{\mathbf{x}_{t-1}, \mathbf{x}_t, t}\!\left[\log p_\theta(\mathbf{x}_{t-1}\mid\mathbf{x}_t, c)\right]
= \mathbb{E}_{\mathbf{x}_{t-1}, \mathbf{x}_t, t}\!\left[\frac{\|\mathbf{x}_{t-1} - \mu_\theta(\mathbf{x}_t, t)\|^2}{2\sigma_t^2 \mathrm{d}t}- C_t\right].
\end{equation}

Recall from Eq.~\eqref{eq:x_tdt_rewrite} that $\mathbf{x}_{t-1} = \mu_\theta(\mathbf{x}_t, t) + \sigma_t \sqrt{\mathrm{d}t}\, \boldsymbol{\epsilon}$ with $\boldsymbol{\epsilon} \sim \mathcal{N}(\mathbf{0}, \mathbf{I})$, so the conditional mean is
\begin{equation}
\bar{\mathbf{x}}_{t-1} := \mathbb{E}[\mathbf{x}_{t-1}\mid \mathbf{x}_t] = \mu_\theta(\mathbf{x}_t, t).
\end{equation}
Plugging this in and pulling the inner expectation through, we obtain
\begin{equation}
\mathcal{H}(p_\theta)
= \mathbb{E}_{t,\mathbf{x}_t}\!\left[\frac{1}{2\sigma_t^2 \mathrm{d}t}\,\mathbb{E}_{\mathbf{x}_{t-1}\mid \mathbf{x}_t}\!\left[\|\mathbf{x}_{t-1} - \bar{\mathbf{x}}_{t-1}\|^2\right]- C_t\right].
\end{equation}

By definition, the expected squared deviation from the conditional mean is the trace of the conditional covariance, which equals the (sum of marginal) conditional variance:
\begin{equation}
\mathbb{E}_{\mathbf{x}_{t-1}\mid \mathbf{x}_t}\!\left[\|\mathbf{x}_{t-1} - \bar{\mathbf{x}}_{t-1}\|^2\right]
= \mathrm{tr}\!\left(\mathrm{Cov}(\mathbf{x}_{t-1}\mid \mathbf{x}_t)\right)
= \mathrm{Var}(\mathbf{x}_{t-1}\mid \mathbf{x}_t).
\end{equation}

Substituting back and absorbing the time-step expectation gives the claimed identity:
\begin{equation}
\boxed{\mathcal{H}(p_\theta) =
\mathbb{E}_{t,\mathbf{x}_t}\!\left[
\frac{\mathrm{Var}(\mathbf{x}_{t-1}\mid \mathbf{x}_t)}
{2\sigma_t^2 \mathrm{d}t}- C_t\right].}
\end{equation}
This completes the proof.
\end{proof}

\subsection{Proof of Corollary~\ref{cor:unified-grpo-main}}
\label{app:proof-unified-grpo}

We first restate the corollary to be proven, then provide the detailed proof.

\begin{corollary*}[Mode-Seeking Optimization under On-Policy Policy Gradient]
For policy gradient methods~\citep{williams1992simple} in Eq.~\eqref{eq:pg}, with $p_{\theta_{\mathrm{old}}}$ and $\mathcal{A}$ fixed within an update, the gradient satisfies:
\begin{equation*}
\nabla_\theta \mathcal{J}(\theta) = \nabla_\theta \Big( -D_{\mathrm{KL}}\big(p_\theta(\cdot|c) \,\|\, p_{\mathcal{A}}(\cdot|c)\big) + D_{\mathrm{KL}}\big(p_\theta(\cdot|c) \,\|\, p_{\theta_{\mathrm{old}}}(\cdot|c)\big) \Big),
\end{equation*}
where all distributions are over trajectories $\bx_{0:T}$, and $p_{\mathcal{A}}(\bx_{0:T}|c)=Z_{\mathcal{A}}(c)^{-1}p_{\theta_{\mathrm{old}}}(\bx_{0:T}|c)\exp(\mathcal{A}(\bx_0,c))$ with $Z_{\mathcal{A}}(c)=\int p_{\theta_{\mathrm{old}}}(\bx_{0:T}|c)\exp(\mathcal{A}(\bx_0,c))\,\mathrm{d}\bx_{0:T}$.
\end{corollary*}

\begin{proof}
Starting from the trajectory-level REINFORCE policy gradient in Eq.~\eqref{eq:pg}:
\begin{equation}
\nabla_\theta \mathcal{J}(\theta) = \mathbb{E}_{\bx_{0:T}\sim p_\theta(\cdot|c)}\Big[\mathcal{A}(\bx_0,c)\sum_{t=1}^{T}\nabla_\theta\log p_\theta(\bx_{t-1}|\bx_t,c)\Big].
\end{equation}
Using the trajectory factorization
\begin{equation}
\log p_\theta(\bx_{0:T}|c)=\log p(\bx_T)+\sum_{t=1}^{T}\log p_\theta(\bx_{t-1}|\bx_t,c),
\end{equation}
and taking $\nabla_\theta$ on both sides, the noise prior $p(\bx_T)$ is independent of $\theta$ so $\nabla_\theta\log p(\bx_T)=0$. This gives
\begin{equation}
\nabla_\theta\log p_\theta(\bx_{0:T}|c)=\sum_{t=1}^{T}\nabla_\theta\log p_\theta(\bx_{t-1}|\bx_t,c).
\end{equation}
Therefore:
\begin{equation}
\label{eq:proof-pg-trajectory}
\nabla_\theta\mathcal{J}(\theta) = \mathbb{E}_{\bx_{0:T}\sim p_\theta}\big[\mathcal{A}(\bx_0,c)\,\nabla_\theta\log p_\theta(\bx_{0:T}|c)\big].
\end{equation}

We rewrite $\mathcal{A}(\bx_0,c)$ as a log-density. Note that the naive Boltzmann form $p_{\mathcal{A}}\propto\exp(\mathcal{A}(\bx_0,c))$ on the trajectory space is not a proper density: since $\mathcal{A}$ depends only on the terminal $\bx_0$, $\int\exp(\mathcal{A}(\bx_0,c))\,\mathrm{d}\bx_{0:T}$ factors through an unbounded integral $\int 1\,\mathrm{d}\bx_{1:T}=\infty$ over the intermediate variables and diverges. Using $p_{\theta_{\mathrm{old}}}$ as the base measure (as in the corollary) supplies a normalized density on $\bx_{1:T}$ and makes $p_{\mathcal{A}}$ a valid probability distribution:
\begin{equation}
\label{eq:pa-explicit-proof}
p_{\mathcal{A}}(\bx_{0:T}|c)
= \frac{1}{Z_{\mathcal{A}}(c)}\,p_{\theta_{\mathrm{old}}}(\bx_{0:T}|c)\exp(\mathcal{A}(\bx_0,c)).
\end{equation}
Inverting this definition gives
\begin{equation}\label{eq:advantage-as-logratio}
\mathcal{A}(\bx_0,c) = \log p_{\mathcal{A}}(\bx_{0:T}|c) - \log p_{\theta_{\mathrm{old}}}(\bx_{0:T}|c) + \log Z_{\mathcal{A}}(c).
\end{equation}
Substituting Eq.~\eqref{eq:advantage-as-logratio} into Eq.~\eqref{eq:proof-pg-trajectory}:
\begin{align}
\nabla_\theta\mathcal{J}(\theta)
&= \mathbb{E}_{p_\theta}\!\big[\log p_{\mathcal{A}}\,\nabla_\theta\log p_\theta\big] - \mathbb{E}_{p_\theta}\!\big[\log p_{\theta_{\mathrm{old}}}\,\nabla_\theta\log p_\theta\big] \notag \\
&\qquad+ \log Z_{\mathcal{A}}(c)\cdot\mathbb{E}_{p_\theta}\!\big[\nabla_\theta\log p_\theta\big].
\end{align}
The third term vanishes by the score-function identity $\mathbb{E}_{p_\theta}[\nabla_\theta\log p_\theta]=\nabla_\theta\,\mathbb{E}_{p_\theta}[1]=0$.

We now apply the log-derivative trick in reverse, which requires the integrand to be $\theta$-independent. Here $\log p_{\theta_{\mathrm{old}}}$ is fixed by construction; for $\log p_{\mathcal{A}}$, the partition function $Z_{\mathcal{A}}(c)=\mathbb{E}_{\bx_{0:T}\sim p_{\theta_{\mathrm{old}}}(\cdot|c)}[\exp(\mathcal{A}(\bx_0,c))]$ is an expectation under the frozen rollout policy and is therefore independent of $\theta$ (had it been taken under $p_\theta$, the trick would acquire an extra $-\nabla_\theta\log Z_{\mathcal{A}}(c)\neq 0$ correction). Both terms thus satisfy
\begin{equation}\label{eq:reverse-logderiv-rigorous}
\mathbb{E}_{p_\theta}\!\big[\log p_{\mathcal{A}}\,\nabla_\theta\log p_\theta\big] = \nabla_\theta\,\mathbb{E}_{p_\theta}[\log p_{\mathcal{A}}],
\qquad
\mathbb{E}_{p_\theta}\!\big[\log p_{\theta_{\mathrm{old}}}\,\nabla_\theta\log p_\theta\big] = \nabla_\theta\,\mathbb{E}_{p_\theta}[\log p_{\theta_{\mathrm{old}}}].
\end{equation}
Hence
\begin{equation}\label{eq:gradJ-as-logratio}
\nabla_\theta\mathcal{J}(\theta) = \nabla_\theta\,\mathbb{E}_{p_\theta}\!\big[\log p_{\mathcal{A}}-\log p_{\theta_{\mathrm{old}}}\big].
\end{equation}

Adding and subtracting $\log p_\theta$ inside each expectation yields the standard cross-entropy / KL identities:
\begin{equation}
\mathbb{E}_{p_\theta}[\log p_{\mathcal{A}}] = -D_{\mathrm{KL}}(p_\theta\|p_{\mathcal{A}}) - \mathcal{H}(p_\theta),
\qquad
\mathbb{E}_{p_\theta}[\log p_{\theta_{\mathrm{old}}}] = -D_{\mathrm{KL}}(p_\theta\|p_{\theta_{\mathrm{old}}}) - \mathcal{H}(p_\theta).
\end{equation}
Subtracting cancels the entropy and gives
\begin{equation}
\mathbb{E}_{p_\theta}\!\big[\log p_{\mathcal{A}}-\log p_{\theta_{\mathrm{old}}}\big] = -D_{\mathrm{KL}}(p_\theta\|p_{\mathcal{A}}) + D_{\mathrm{KL}}(p_\theta\|p_{\theta_{\mathrm{old}}}).
\end{equation}
Combining with Eq.~\eqref{eq:gradJ-as-logratio} yields the final gradient formula:
\begin{equation}
\boxed{\nabla_\theta\mathcal{J}(\theta) = \nabla_\theta\Big(-D_{\mathrm{KL}}\big(p_\theta(\cdot|c)\,\|\,p_{\mathcal{A}}(\cdot|c)\big) + D_{\mathrm{KL}}\big(p_\theta(\cdot|c)\,\|\,p_{\theta_{\mathrm{old}}}(\cdot|c)\big)\Big).}
\end{equation}
This completes the proof.
\end{proof}

\subsection{Derivation for Remark~\ref{prop:perceptual-entropy-variance}}
\label{app:proof-perceptual-entropy-variance}

We derive the first-order relationship in Remark~\ref{prop:perceptual-entropy-variance} in the main text, showing how perceptual entropy tracks the conditional feature-space variance of the reverse transition under a local smooth-encoder approximation. As in the main text, we use $\mathrm{Var}(\mathbf{z}\mid \mathbf{x})$ as a shorthand for $\mathrm{tr}(\mathrm{Cov}(\mathbf{z}\mid \mathbf{x}))$, the (sum of marginal) conditional variance of $\mathbf{z}$ given $\mathbf{x}$. We restate the claim for convenience.

\begin{remark*}[Connection to Conditional Feature-Space Variance]
Let $\mathbf{z}_{t-1}=\phi(\mathbf{x}_{t-1})$ and $\mathbf{m}=\mu_{\theta_{\mathrm{old}}}(\mathbf{x}_t)$ as in Definition~\ref{def:perceptual-entropy}, with all expectations conditioned on $\mathbf{x}_t$ taken under the rollout transition $p_{\theta_{\mathrm{old}}}$. Assume that, within a local reverse step, the perceptual encoder is approximately mean-preserving, i.e., $\mathbb{E}\!\left[\mathbf{z}_{t-1}\mid \mathbf{x}_t\right]\approx \phi(\mathbf{m})$. Then the perceptual entropy in Definition~\ref{def:perceptual-entropy} tracks the expected conditional feature-space variance:
\begin{equation*}
\mathcal{H}_{\text{perc}}(p_\theta) \approx
\mathbb{E}_{t,\mathbf{x}_t}\!\left[
\frac{\mathrm{Var}(\mathbf{z}_{t-1}\mid \mathbf{x}_t)}
{2\sigma_t^2 \mathrm{d}t}\right] - C_t,
\end{equation*}
\noindent where $C_t = -\tfrac{d}{2}\log(2\pi\sigma_t^2\,\mathrm{d}t)$ is independent of $\theta$.
\end{remark*}

\begin{proof}[Derivation]
\textbf{Local linearization of $\phi$.} Under the rollout transition, $\mathbf{x}_{t-1}\mid\mathbf{x}_t \sim \mathcal{N}(\mathbf{m},\,\sigma_t^2\mathrm{d}t\,\mathbf{I})$ with $\mathbf{m}=\mu_{\theta_{\mathrm{old}}}(\mathbf{x}_t)$, so $\mathbf{m}$ is exactly the conditional mean of $\mathbf{x}_{t-1}$ given $\mathbf{x}_t$. For a frozen encoder $\phi$, the second-order Taylor expansion around $\mathbf{m}$ gives
\begin{equation}
\phi(\mathbf{x}_{t-1})
= \phi(\mathbf{m}) + J_{\phi}(\mathbf{m})(\mathbf{x}_{t-1}-\mathbf{m})
+ \tfrac{1}{2}(\mathbf{x}_{t-1}-\mathbf{m})^\top H_\phi(\mathbf{m})(\mathbf{x}_{t-1}-\mathbf{m})
+ \mathcal{O}(\|\mathbf{x}_{t-1}-\mathbf{m}\|^3).
\end{equation}
Since $\mathbb{E}[\mathbf{x}_{t-1}-\mathbf{m}\mid\mathbf{x}_t]=\mathbf{0}$ holds exactly under $p_{\theta_{\mathrm{old}}}$, the first-order term vanishes upon taking expectations. The leading residual is a second-order Jensen gap:
\begin{equation}
\mathbb{E}_{\mathbf{x}_{t-1}\mid\mathbf{x}_t}[\phi(\mathbf{x}_{t-1})]
= \phi(\mathbf{m}) + \tfrac{1}{2}\sigma_t^2\mathrm{d}t\,\operatorname{tr}(H_\phi(\mathbf{m})) + \mathcal{O}((\sigma_t^2\mathrm{d}t)^2).
\end{equation}
When the encoder is locally smooth, i.e., $\sigma_t^2\mathrm{d}t\,\|H_\phi\|$ is small, this yields the mean-preserving approximation $\mathbb{E}[\phi(\mathbf{x}_{t-1})\mid\mathbf{x}_t]\approx\phi(\mathbf{m})$.

\textbf{Perceptual entropy as a conditional residual term.} Starting from Definition~\ref{def:perceptual-entropy}, the perceptual entropy proxy is
\begin{equation}
\mathcal{H}_{\text{perc}}(p_\theta) =
\mathbb{E}_{t,k}\!\left[-\log p_\theta^{\mathrm{perc}}(\mathbf{x}_{t-1}^k\mid\mathbf{x}_t^k,c)\right],
\end{equation}
where $\mathbf{z}_{t-1} = \phi(\mathbf{x}_{t-1})$ denotes samples in the perceptual space. Defining $p_\theta^{\mathrm{perc}}$ directly as a feature-space scoring distribution omits the Jacobian of $\phi$, following score distillation methods in which a frozen vision or diffusion model supplies a surrogate optimization signal in feature space~\citep{poole2022dreamfusion,wang2023prolificdreamer,lin2023magic3d,wang2024animatabledreamer}.

Under the Gaussian score in Definition~\ref{def:perceptual-entropy}, $p_\theta^{\mathrm{perc}}$ has mean $\phi(\mathbf{m})$ and isotropic covariance $\sigma_t^2\,\mathrm{d}t\cdot\mathbf{I}$, so its log-density is
\begin{equation}
\log p_\theta^{\mathrm{perc}}
= -\frac{\|\mathbf{z}_{t-1} - \phi(\mathbf{m})\|^2}{2\sigma_t^2 \mathrm{d}t} + C,
\qquad C = -\tfrac{d}{2}\log(2\pi\sigma_t^2 \mathrm{d}t).
\end{equation}
Taking the negative expectation, the perceptual entropy proxy at fixed $t$ becomes
\begin{equation}
\mathcal{H}_{\text{perc}}(p_\theta\mid t) =
\frac{1}{2\sigma_t^2 \mathrm{d}t}\,
\mathbb{E}_{\mathbf{x}_t}\,
\mathbb{E}_{\mathbf{z}_{t-1}\mid \mathbf{x}_t}\!
\left[\|\mathbf{z}_{t-1} - \phi(\mathbf{m})\|^2\right] - C.
\end{equation}

\textbf{Approximation to conditional variance.} Combining the mean-preserving result $\mathbb{E}[\mathbf{z}_{t-1}\mid\mathbf{x}_t]= \phi(\mathbf{m}) + \mathcal{O}(\sigma_t^2\mathrm{d}t)$ with the bias-variance decomposition
\begin{equation*}
\mathbb{E}\!\left[\|\mathbf{z}_{t-1}-\phi(\mathbf{m})\|^2\mid\mathbf{x}_t\right]
= \operatorname{tr}\!\left(\mathrm{Cov}(\mathbf{z}_{t-1}\mid\mathbf{x}_t)\right)
+ \|\mathbb{E}[\mathbf{z}_{t-1}\mid\mathbf{x}_t]-\phi(\mathbf{m})\|^2,
\end{equation*}
the bias term is $\mathcal{O}((\sigma_t^2\mathrm{d}t)^2)$ and is negligible when the encoder is locally smooth, yielding
\begin{equation}
\mathbb{E}_{\mathbf{z}_{t-1}\mid\mathbf{x}_t}\!\left[\|\mathbf{z}_{t-1}-\phi(\mathbf{m})\|^2\right]
\approx \operatorname{tr}\!\left(\mathrm{Cov}(\mathbf{z}_{t-1}\mid\mathbf{x}_t)\right)
= \mathrm{Var}(\mathbf{z}_{t-1}\mid\mathbf{x}_t).
\end{equation}
Substituting back and taking the expectation over timesteps yields
\begin{equation}
\boxed{\mathcal{H}_{\text{perc}}(p_\theta) \approx
\mathbb{E}_{t,\mathbf{x}_t}\!\left[
\frac{\mathrm{Var}(\mathbf{z}_{t-1}\mid \mathbf{x}_t)}
{2\sigma_t^2 \mathrm{d}t}\right] - C_t,}
\end{equation}
which shows that perceptual entropy tracks the conditional feature-space variance of the reverse transition up to a timestep-dependent affine transformation.

For a rollout group $\{\mathbf{z}_{t-1}^k\}_{k=1}^K = \{\phi(\mathbf{x}_{t-1}^k)\}_{k=1}^K$, the corresponding empirical proxy at timestep $t$ centers each residual at the rollout-policy transition mean $\phi(\mathbf{m}^k)$ with $\mathbf{m}^k=\mu_{\theta_{\mathrm{old}}}(\mathbf{x}_t^k,t,c)$:
\begin{equation}
\widehat{\mathrm{Var}}(t)
= \frac{1}{K}\sum_{k=1}^K \|\mathbf{z}_{t-1}^k - \phi(\mathbf{m}^k)\|^2,
\end{equation}
which gives a sample estimate of $\mathbb{E}_{\mathbf{x}_t}[\mathrm{Var}(\mathbf{z}_{t-1}\mid\mathbf{x}_t)]$ used in the reward shaping of Sec.~\ref{sec:perceptual-shaping}.
\end{proof}

\section{Covariance-Aware Entropy Control for Flow GRPO}\label{app:cov-flow}

Natural policy gradient is rarely used in post-training of generative models due to its second-order optimization cost. However, its target formulation with KL constraints shares the same spirit with TRPO and PPO. Motivated by {the log-probability--advantage covariance principle discussed in prior LLM RL analyses~\citep{cui2025entropy}}, we construct a flow analogue for GRPO using reverse-transition likelihoods and derive covariance-aware constraints.

\subsection{Covariance-Based Sample Identification}

For a batch of $N$ transitions $\{(\bx_{t-1}^{i}, \bx_t^{i}, c^i)\}_{i=1}^{N}$ sampled from the exploration stage, we measure how strongly the log-probability and advantage co-vary for each sample. Specifically, for each transition $i$, we compute the centered cross-product:
\begin{equation}\label{eq:cov}
    \mathrm{Cov}(i) = \left(\log p_\theta(\bx_{t-1}^{i}\mid \bx_t^{i},c^i) - \bar{\ell}\right) \cdot \left(\mathcal{A}_i - \bar{\mathcal{A}}\right),
\end{equation}
where $\bar{\ell} = \frac{1}{N}\sum_{j=1}^{N}\log p_\theta(\bx_{t-1}^{j}\mid \bx_t^{j},c^j)$ is the batch mean log-probability and $\bar{\mathcal{A}} = \frac{1}{N}\sum_{j=1}^{N}\mathcal{A}_j$ is the batch mean advantage. Here, $p_\theta(\bx_{t-1}\mid \bx_t,c^i) = \mathcal{N}(\mu_\theta(\bx_t,t,c^i), \sigma_t^2\,\mathrm{d}t\,\mathbf{I})$ is the conditional Gaussian policy with log-probability $\log p_\theta(\bx_{t-1}\mid \bx_t,c^i) = -\frac{\|\bx_{t-1}-\mu_\theta(\bx_t,t,c^i)\|^2}{2\sigma_t^2\,\mathrm{d}t}+C_t$, and $\mathcal{A}_i$ denotes the GRPO group-normalized advantage. The empirical covariance in Eq.~\eqref{eq:cov} instantiates the analogous covariance between transition log-likelihood and advantage in the flow setting.

A large positive $\mathrm{Cov}(i)$ indicates that transitions with high log-probability receive high advantages (and vice versa), signaling a strong positive correlation between policy likelihood and performance. Following {the same covariance interpretation}, such transitions are treated as entropy-reducing samples during gradient optimization. The key insight is that by identifying and selectively moderating updates on these high-covariance samples, we can control the rate of entropy collapse while maintaining reward optimization.

\subsection{Entropy-Aware Constraint via Gradient Clipping}

The Clip-Cov strategy leverages covariance scores to selectively suppress entropy-accelerating gradients. Given covariance bounds $\omega_{\text{low}}, \omega_{\text{high}} \in \mathbb{R}$ and a clipping ratio $r \in (0, 1)$, we randomly select transitions whose covariance falls within the specified range:
\begin{equation}
\mathcal{I}_{\text{clip}} \sim \text{Uniform}\Big(\{i \mid \mathrm{Cov}(i) \in [\omega_{\text{low}}, \omega_{\text{high}}]\}, \lfloor r \cdot N\rfloor\Big).
\end{equation}
The Clip-Cov objective then detaches these selected transitions from gradient flow:
\begin{equation}\label{eq:clip-cov}
\mathcal{J}_{\text{\texttt{Clip-Cov}}}(\theta) = \mathcal{J}_{\mathrm{GRPO}}(\theta)\,\mathbf{1}_{\{i \notin \mathcal{I}_{\text{clip}}\}}.
\end{equation}
By zeroing out gradients on transitions with covariance in $[\omega_{\text{low}}, \omega_{\text{high}}]$, we prevent the most aggressive entropy-reducing updates while preserving normal gradient flow on the majority of the batch. This achieves a pragmatic balance: entropy-accelerating samples are suppressed, but the policy still makes meaningful progress on reward optimization. The clipping ratio $r$ and covariance bounds $\omega_{\text{low}}, \omega_{\text{high}}$ serve as tunable parameters to control the strength of entropy regularization. In practice, $\omega_{\text{high}}$ is typically set to a very large value (effectively $+\infty$) since we aim to suppress the most aggressive entropy-reducing updates.

\paragraph{Integrating Perceptual Entropy into Clip-Cov.}
To incorporate perceptual entropy $\mathcal{H}_{\text{perc}}$ into the Clip-Cov framework, we replace the standard covariance computation with a perceptual covariance metric. Specifically, we compute the perceptual log-probability $\log p_{\text{perc}}(\bx_{t-1}^{i}\mid \bx_t^{i},c^i)$ for each generated sample, and define the perceptual covariance as:
\begin{equation}
\mathrm{Cov}_{\mathrm{perc}}(i) = \left(\log p_{\text{perc}}(\bx_{t-1}^{i}\mid \bx_t^{i},c^i) - \bar{\ell}_{\text{perc}}\right) \cdot \left(\mathcal{A}_i - \bar{\mathcal{A}}\right),
\end{equation}
where $\bar{\ell}_{\text{perc}} = \frac{1}{N}\sum_{j=1}^{N}\log p_{\text{perc}}(\bx_{t-1}^{j}\mid \bx_t^{j},c^j)$ is the mean perceptual log-probability and $\bar{\mathcal{A}} = \frac{1}{N}\sum_{j=1}^{N}\mathcal{A}_j$ is the batch mean advantage. Transitions with large $\mathrm{Cov}_{\mathrm{perc}}(i)$ indicate samples that achieve high advantage yet receive high perceptual likelihood, signaling potential diversity collapse in perceptual space. The modified Clip-Cov objective then selects samples based on $\mathrm{Cov}_{\mathrm{perc}}(i)$ instead of $\mathrm{Cov}(i)$, thereby directly targeting perceptual diversity preservation.

\subsection{Entropy-Aware Constraint via KL Regularization}

An alternative approach applies a KL penalty directly to high-covariance transitions. Given a threshold ratio $k \in (0, 1)$ and penalty strength $\beta > 0$, we identify the top-$k$ proportion of transitions ranked by covariance:
\begin{equation}
\mathcal{I}_{\text{KL}} = \Big\{i \mid \mathrm{Rank}(\mathrm{Cov}(i)) \leq k \cdot N \Big\},
\end{equation}
where $\mathrm{Rank}(\mathrm{Cov}(i))$ denotes the rank in descending covariance order. For these entropy-critical transitions, we impose a KL constraint:
\begin{equation}\label{eq:kl-cov}
\mathcal{J}_{\text{\texttt{KL-Cov}}}(\theta) = \mathcal{J}_{\mathrm{GRPO}}(\theta) - \beta\,D_{\mathrm{KL}}\big(p_{\theta_{\mathrm{old}}}(\bx_{t-1}^{i}\mid \bx_t^{i},c^i) \,\|\, p_\theta(\bx_{t-1}^{i}\mid \bx_t^{i},c^i)\big)\,\mathbf{1}_{\{i \in \mathcal{I}_{\text{KL}}\}}.
\end{equation}
For Gaussian policies with shared covariance $\sigma_t^2\,\mathrm{d}t$, this KL divergence reduces to:
\begin{equation}\label{eq:kl-gaussian}
D_{\mathrm{KL}}\big(p_{\theta_{\mathrm{old}}} \| p_\theta\big) = \frac{\|\mu_{\theta_{\mathrm{old}}}^{i} - \mu_\theta^{i}\|^2}{2\sigma_t^2\,\mathrm{d}t}.
\end{equation}
The KL-Cov approach penalizes large policy changes on entropy-critical samples, thereby constraining the magnitude of entropy-reducing updates. The penalty coefficient $\beta$ governs the trade-off between reward optimization and entropy preservation. In contrast to Clip-Cov, which discretely suppresses gradients, KL-Cov applies a continuous penalty that allows some flexibility while still dampening aggressive updates. Together, these two strategies provide complementary mechanisms to moderate entropy decrease by directly targeting samples identified as entropy accelerators through the flow covariance criterion.

\paragraph{Integrating Perceptual Entropy into KL-Cov.}
To integrate perceptual entropy into the KL-Cov framework, we replace the sample ranking criterion with perceptual covariance. Instead of ranking by $\mathrm{Cov}(i)$, we identify the top-$k$ proportion of transitions ranked by perceptual covariance:
\begin{equation}
\mathcal{I}_{\text{KL-perc}} = \Big\{i \mid \mathrm{Rank}(\mathrm{Cov}_{\mathrm{perc}}(i)) \leq k \cdot N \Big\},
\end{equation}
where $\mathrm{Rank}(\mathrm{Cov}_{\mathrm{perc}}(i))$ denotes the rank in descending perceptual covariance order. The modified objective applies the same KL penalty to these selected transitions:
\begin{equation}
\mathcal{J}_{\text{\texttt{KL-Cov-perc}}}(\theta) = \mathcal{J}_{\mathrm{GRPO}}(\theta) - \beta\,D_{\mathrm{KL}}\big(p_{\theta_{\mathrm{old}}}(\bx_{t-1}^{i}\mid \bx_t^{i},c^i) \,\|\, p_\theta(\bx_{t-1}^{i}\mid \bx_t^{i},c^i)\big)\,\mathbf{1}_{\{i \in \mathcal{I}_{\text{KL-perc}}\}}.
\end{equation}
By selecting samples based on perceptual covariance, the KL penalty is applied to transitions that most strongly contribute to perceptual entropy decrease, thereby directly targeting diversity preservation in perceptual space.

\section{Broader Impact}
\label{app:broader_impact}

\paragraph{Positive impacts.}
This work improves diversity preservation in RLHF-aligned text-to-image generation, enabling models to produce a broader range of high-quality outputs for a given prompt. This benefits creative applications such as design assistance, content creation, and personalized media generation, where coverage across diverse styles and contexts is essential.

\paragraph{Negative impacts.}
Improved image generation quality and diversity may lower the barrier to producing misleading or harmful visual content, including deepfakes and synthetic disinformation. Moreover, while pretrained visual generative models typically undergo strict data safety auditing to suppress violent or pornographic outputs, visual post-training and diversity-promoting optimization can weaken these built-in safeguards and reintroduce harmful visual content. Our method is a fine-tuning technique applied on top of existing base models (e.g., FLUX.dev) that carry their own usage policies, and practical deployments should pair reward and diversity optimization with dedicated safety filters.